\documentclass{article}

\usepackage[main,final]{neurips_2026}

\usepackage[utf8]{inputenc} 
\usepackage[T1]{fontenc}    
\usepackage{hyperref}       
\usepackage{url}            
\usepackage{booktabs}       
\usepackage{amsfonts}       
\usepackage{nicefrac}       
\usepackage{microtype}      
\usepackage{xcolor}         

\usepackage{enumitem}  
\usepackage{graphicx}
\usepackage{subcaption}
\usepackage{siunitx}

\usepackage{multirow} 
\usepackage{amsmath}
\usepackage{amssymb}
\usepackage{mathtools}
\usepackage{amsthm}
\usepackage{bbm}
\usepackage[capitalize,noabbrev]{cleveref}

\usepackage{algorithm}
\usepackage{algpseudocode}
\theoremstyle{plain}
\newtheorem{theorem}{Theorem}[section]
\newtheorem{proposition}[theorem]{Proposition}
\newtheorem{lemma}[theorem]{Lemma}
\newtheorem{corollary}[theorem]{Corollary}
\theoremstyle{definition}
\newtheorem{definition}[theorem]{Definition}
\newtheorem{assumption}[theorem]{Assumption}
\theoremstyle{remark}
\newtheorem{remark}[theorem]{Remark}

\definecolor{custombeige}{HTML}{F4C7A8}
\definecolor{customtail}{HTML}{4fb9af}

\usepackage{bm}
\usepackage{braket}
\usepackage[T1]{fontenc}

\newcommand{\vx}{\bm{x}}

\definecolor{blue1}{RGB}{204,229,255}
\definecolor{blue2}{RGB}{0,76,153}

\usepackage{pifont}

\usepackage{tcolorbox}
\tcbuselibrary{skins, breakable}
\tcbset{
    boxrule=0pt,
    colframe=gray!90,
    colback=gray!20,
    coltitle=white,
    enhanced,
    boxsep=5pt,
    arc=2mm}

\title{The Sign Code: The Hidden Binary Nature of Deep Networks}

\author{%
  Niclas G\"oring$^{1}$\thanks{Corresponding author: \texttt{niclas.goring@physics.ox.ac.uk}}\quad
  Shuofeng Zhang$^{1}$\quad
  Roi Holtzman$^{1}$\\
  \bfseries Yoonsoo Nam$^{3}$\quad
  Ard A.~Louis$^{1}$\\[4pt]
  $^{1}$Rudolf Peierls Centre for Theoretical Physics, University of Oxford, Oxford, UK\\
  $^{2}$Kim Jaechul Graduate School of AI, KAIST, Seoul, South Korea
}

\begin{document}

\maketitle

\begin{abstract}

Every neuron in a trained network outputs a number that can be separated into its sign (positive or negative) and its magnitude (absolute size). We call the set of all neuron signs the \textit{sign code}. This raises a key question: does classification depend more on the sign code or on the magnitudes? For overparameterized vision classifiers with sign-preserving activations, the answer is overwhelmingly the sign code. A $k$-nearest neighbor classifier operating only on this binary sign code achieves accuracy comparable to the full model across different architectures, activation functions, and datasets. Moreover, replacing all magnitudes in every layer with random noise while keeping the original signs barely changes residual network performance. In contrast, flipping just 20\% of the signs reduces accuracy to chance.
Models generalize well when same-class inputs share similar sign codes. 
When this overlap is disrupted by a regularizer, training accuracy remains high, but test performance collapses.
A margin decomposition helps explain these results: the sign-only contribution to each logit grows as $\sqrt{d}$ in network width, dominating the magnitude residual. The natural unit of analysis for these networks may not be through their continuous weight space but instead through their discrete sign codes.

\end{abstract}

\section{Introduction}
\label{sec:intro}

ReLU networks admit a clean piecewise-linear decomposition: each input activates a subset of neurons, and within the resulting region the network acts as an affine map. This combinatorial structure has been a productive lens on deep learning, underpinning analyses of linear-region counts \citep{pascanu2014number,montufar2014number,raghu2017expressive,serra2018bounding,hanin2019complexity,gamba2022regions}, spline and max-affine views \citep{balestriero2018spline}, and gate- and path-based studies of expressivity and generalization \citep{neyshabur2015pathsgd,neyshabur2015norm,fiat2019decoupling,lakshminarayanan2020neural,gonon2024pathnorm,lebeurrier2026pathconditioned}. But modern networks rarely use ReLU, while SiLU, GELU, and tanh have no piecewise-linear structure to exploit so that the tools built around ReLU do not transfer. We ask whether the combinatorial object underlying these ReLU analyses, which side of each neuron's hyperplane the input falls on, survives the move to smooth activations, and whether it remains the dominant carrier of classification signal.

For any sign-preserving activation $\varphi$ (ReLU, SiLU, GELU, tanh), each neuron output decomposes exactly as $h_i(x) = S_i(x) \cdot M_i(x)$, where $S_i \in \{-1, +1\}$ is a binary sign and $M_i = |h_i|$ is a non-negative magnitude (\cref{fig:hero}a,b). The sign identifies which combinatorial region of representation space the input occupies; the magnitude specifies where within that region. We call the collection $\{S_i(x)\}$ the \emph{sign code} of the network, and ask whether it alone carries the essential information for classification. Across architectures, activations, and datasets, our experiments reveal a sharp asymmetry: the sign code is sufficient to recover full-model accuracy and causally dominant under intervention, while magnitudes are largely redundant (\cref{fig:hero}c).

\begin{figure}[htb!]
  \centering
   \includegraphics[width=\textwidth]{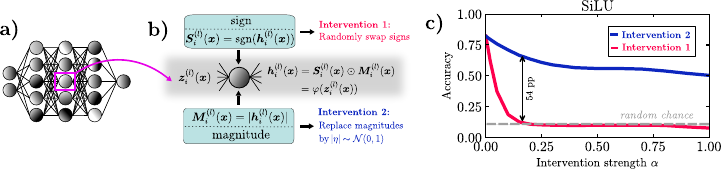}
  \caption{\textbf{a)}~Every neuron in a feedforward network computes a pre-activation $z_j^{(\ell)}$, which the activation function maps to $h_j^{(\ell)} = \varphi(z_j^{(\ell)})$.
  \textbf{b)}~For any sign-preserving activation, this output decomposes as $h_j = S_j \odot M_j$: a binary sign $S_j \in \{-1,+1\}$ encoding \emph{which} side of the neuron's hyperplane the input falls on, and a non-negative magnitude $M_j = |h_j|$ encoding \emph{how far}.
  We test the causal role of each component with two channel-isolating interventions: randomly flipping signs while preserving magnitudes, or replacing magnitudes while enforcing the  sign code.
  \textbf{c)}~Flipping just 20\% of sign bits collapses accuracy to chance (\textcolor{red}{red}); replacing \emph{every} magnitude in every layer preserves far more classification performance (\textcolor{blue}{blue}), and is nearly lossless for residual networks. The forward computation is dominated by binary routing decisions rather than raw activation strengths. }
  \label{fig:hero}
\end{figure}

Our \textbf{contributions} are:
\begin{enumerate}[leftmargin=20pt,itemsep=2pt,topsep=2pt]
\item[(i)] \textbf{The sign code is informationally sufficient for classification.}
A $k$-nearest-neighbour classifier on the last-layer binary sign code, with no access to magnitudes, weights, or logits, matches
full model accuracy in the overparameterized regime, across
architectures, activations, and datasets. We give a formal sufficient
certificate for this transition
(\cref{sec:sufficient,sec:theory}).

\item[(ii)] \textbf{Networks are robust to magnitude perturbations but not to sign code perturbations.}
We compare two interventions that isolate the sign and magnitude channels: destroying signs while preserving magnitudes is catastrophic, whereas randomizing magnitudes while preserving the clean sign trajectory is much less harmful.
We check this asymmetry across four activations and two architecture families (\cref{sec:causal}, \cref{fig:causal}).

\item[(iii)] \textbf{Class-coherent sign geometry is tied to generalization.}
Models generalize when same-class inputs share similar sign codes. A regularizer that actively repels within-class sign codes maintains high training accuracy while sharply degrading test accuracy, suggesting that class-coherent sign structure is a necessary ingredient for generalization in these models (\cref{sec:gen}).
\end{enumerate}


Taken together, these results suggest that the natural unit of analysis for these networks may not be their continuous weight space but their \emph{discrete sign code}. 

\subsection{Related Work}
\label{sec:related}

A large body of work studies ReLU networks through their input-space
partitions, counting linear regions and bounding expressive gains from
depth
\citep{pascanu2014number,montufar2014number,raghu2017expressive,serra2018bounding,hanin2019complexity,gamba2022regions}
or interpreting deep networks as input-dependent affine maps indexed by
activation patterns \citep{balestriero2018spline}. Binary activation
patterns have been examined as hash representations \citep{he2021neural},
learned via specialised objectives \citep{park2021unsupervised}, and used
for verification \citep{geng2023towards}; on the theoretical side,
path-norm, path-kernel, and gate-decoupling analyses elevate gates to a
primary object of study
\citep{neyshabur2015pathsgd,neyshabur2015norm,lakshminarayanan2020neural,fiat2019decoupling,gonon2024pathnorm,lebeurrier2026pathconditioned}.
The gate-decoupling line is  closest to our work: \citet{fiat2019decoupling} separates gating from linearity and \citet{lakshminarayanan2020neural} builds kernels from the active subnetwork.
Both treat gates as a useful analytical device for ReLU; we make three stronger claims: informational sufficiency, causal dominance under intervention, and a link between sign geometry and generalization (\cref{sec:sufficient,sec:causal,sec:gen}). We establish them across all sign-preserving activations, including smooth ones with no piecewise-linear structure.
Neural collapse \citep{papyan2020prevalence,han2022neural} describes
convergence of continuous features toward class means; our results hold
outside this regime (see \cref{appx_NC} for a thorough comparison). A
comprehensive discussion of related work is given in
Appendix~\ref{appx:related}.

\section{Background: The Sign--Magnitude Decomposition}
\label{sec:decomposition}

\subsection{Setup}
\label{sec:setup}

We consider an $L$-layer feedforward network $f\colon \mathcal{X} \to \mathbb{R}^K$ (FFNN) with hidden representations $\mathbf{h}^{(0)}(x) = x$ and, for $\ell = 1,\ldots,L$,
\begin{equation}
  \mathbf{z}^{(\ell)}
  = W^{(\ell)}\,\mathbf{h}^{(\ell-1)} + \mathbf{b}^{(\ell)}\,,
  \qquad
  \mathbf{h}^{(\ell)}
  = \varphi\!\bigl(\mathbf{z}^{(\ell)}\bigr)\,,
  \label{eq:ffnn}
\end{equation}
where $W^{(\ell)} \in \mathbb{R}^{d_\ell \times d_{\ell-1}}$, $\mathbf{b}^{(\ell)} \in \mathbb{R}^{d_\ell}$, and $\varphi$ is a coordinatewise activation function.
The output layer produces logits $f(x) = W^{(L+1)}\mathbf{h}^{(L)} + \mathbf{b}^{(L+1)}$.

Each neuron $j$ in layer $\ell$ computes a pre-activation $z_j^{(\ell)} = \mathbf{w}_j^\top \mathbf{h}^{(\ell-1)} + b_j$.
The zero set $\{z_j^{(\ell)} = 0\}$ is a hyperplane in the representation space of the preceding layer.
A layer of $d_\ell$ neurons therefore partitions its input space by $d_\ell$ hyperplanes, assigning each input a binary code, one bit per neuron, that records which side of every hyperplane it falls on.
The question is whether the activation function preserves this partition.

\subsection{Sign-preserving activations}
\label{sec:sign-pres}

\begin{definition}[Sign-preserving activation]
\label{def:sign-preserving}
An activation $\varphi\colon \mathbb{R}\to\mathbb{R}$ is \textbf{sign-preserving} if\;
$\operatorname{sgn}\!\bigl(\varphi(z)\bigr) = \operatorname{sgn}(z)$
for every $z \neq 0$, where $\operatorname{sgn}(t) = +1$ for $t > 0$ and $-1$ otherwise.
\end{definition}
ReLU, SiLU, GELU, and tanh are all sign-preserving which will be the central point of study.
A ReLU neuron zeroes out the negative half-space entirely, producing the sharpest possible partition; smooth activations such as SiLU and GELU attenuate rather than annihilate the negative side, but the sign of the output still mirrors the sign of the pre-activation. We discuss non sign-preserving  activations such as softplus in Appendix \ref{app:non-sign-preserving}.
\subsection{The decomposition}
\label{sec:decomp-def}

Every sign-preserving activation allows us to split each neuron's output into two factors: a binary routing decision and a continuous magnitude.

\begin{definition}[Sign--magnitude decomposition]
\label{def:sign-mag}
Let $\varphi$ be sign-preserving and let $h_j^{(\ell)}(\vx) = \varphi\!\bigl(z_j^{(\ell)}(\vx)\bigr)$ be the output of neuron $j$ at layer $\ell$.
The \textbf{sign} and \textbf{magnitude} of neuron $j$ are
\begin{equation}
  S_j^{(\ell)}(\vx)
  \;=\; \operatorname{sgn}\!\bigl(h_j^{(\ell)}(\vx)\bigr)
  \;\in\; \{-1,+1\}\,,
  \qquad
  M_j^{(\ell)}(\vx)
  \;=\; \bigl|h_j^{(\ell)}(\vx)\bigr|
  \;\in\; \mathbb{R}_{\geq 0}\,,
  \label{eq:neuron-decomp}
\end{equation}
so that the neuron's output is
\begin{equation}
\colorbox{teal!15}{$\displaystyle
  h_j^{(\ell)}(\vx)
  \;=\; S_j^{(\ell)}(\vx) \;\cdot\; M_j^{(\ell)}(\vx)\,.
$}
\label{eq:hero}
\end{equation}
\end{definition}
In vector form, the full hidden representation decomposes as $\mathbf{h}^{(\ell)} = \mathbf{S}^{(\ell)} \odot \mathbf{M}^{(\ell)}$.
The \textbf{sign code} of the network is the concatenation of all per-layer sign vectors:
\begin{equation}
  \mathbf{S}(\vx)
  = \bigl(\mathbf{S}^{(1)}(\vx),\;\ldots,\;\mathbf{S}^{(L)}(\vx)\bigr)
  \;\in\; \{-1,+1\}^{D}\,,
  \qquad
  D = \textstyle\sum_{\ell=1}^{L} d_\ell\,.
  \label{eq:full-sign-code}
\end{equation}
Because $\varphi$ is sign-preserving, $S_j^{(\ell)} = \operatorname{sgn}(z_j^{(\ell)})$: the sign code is determined entirely by the pre-activation signs and is therefore \emph{identical} across all sign-preserving activations sharing the same weights.
Intuitively, $\mathbf{S}(x)$ records \emph{which} combinatorial region of representation space the input occupies, which side of every learned hyperplane it falls on, while $\mathbf{M}(x)$ records \emph{where} within that region. The sign code recovers the classical ReLU gate pattern $\mathbf{g}^{(\ell)} = \mathbf{1}[z^{(\ell)} > 0]$ under the relabelling $+1 \leftrightarrow \text{active}$, $-1 \leftrightarrow \text{inactive}$. The sign--magnitude decomposition for smooth activations (SiLU, GELU, tanh) captures the same combinatorial structure without requiring exact piecewise-affine behaviour.  We discuss the natural generalization of this decomposition to convolutional layers as well as residual blocks in \cref{appx_signextr}.

\section{Experiment: Sign Codes Are Informationally Sufficient}
\label{sec:sufficient}

The sign--magnitude decomposition (Definition~\ref{def:sign-mag}) splits every hidden representation into a discrete sign code and a continuous magnitude.
We begin with a natural question:

\begin{tcolorbox}[
  colback=custombeige!45, colframe=black,
  boxsep=2pt, left=4pt, right=4pt, top=4pt, bottom=2pt
]
\textbf{Question 1:} Does the sign code alone, without any magnitude information, already contain enough structure to predict class labels?
\end{tcolorbox}


\noindent
We test this by building the simplest possible classifier that operates \emph{only} on the binary sign code of a single hidden layer, and comparing its accuracy to that of the full network.

\subsection{Experimental setup: Sign-$k$NN}
\label{sec:signknn}

For each trained model we extract the sign code $\mathbf{S}^{(L)}(x) \in \{-1,+1\}^{d_L}$ of the \emph{last hidden layer only}: $\operatorname{sgn}\!\bigl(\mathbf{h}^{(L)}(x)\bigr)$ for FFNNs, the output code of the last \texttt{BasicBlock} for ResNets, and the output code of the final transformer block for ViTs.
This is a strict subset of the full-network sign code~\eqref{eq:full-sign-code}.
Full extraction details are given in Appendix~\ref{app:extraction}.
We equip the sign code space with the normalised Hamming distance
\begin{equation}
  d_H\!\bigl(\mathbf{S}^{(L)}(x),\, \mathbf{S}^{(L)}(x')\bigr)
  \;=\;
  \frac{1}{d_L}
  \sum_{j=1}^{d_L}
  \mathbf{1}\!\bigl[S^{(L)}_j(x) \neq S^{(L)}_j(x')\bigr]\,,
  \label{eq:hamming}
\end{equation}
and classify each test point by majority vote among its $k$-nearest training-set neighbours in this distance.
The resulting classifier, sign-$k$NN$_k$, uses only the binary sign code of a single layer. It has no access to activation magnitudes, to the sign codes of any earlier layer, or to the network's weights or logits. To quantify how well the sign code captures the network's classification, we define the \textbf{sign gap}:
\begin{equation}
  \Delta_{\mathrm{sgn}}
  \;=\;
  \bigl|\,\mathrm{err}_{\mathrm{SignKNN}} - \mathrm{err}_{\mathrm{model}}\,\bigr|\,,
  \label{eq:sign-gap}
\end{equation}
the absolute difference between the sign-$k$NN test error and the full model's test error.
When $\Delta_{\mathrm{sgn}} \approx 0$, the binary sign code is sufficient to match the network's test error; when $\Delta_{\mathrm{sgn}}$ is large, information outside the last-layer sign code remains relevant to the model's predictions.

\subsection{Results:  The sign gap across activation functions and architectures}

Figure~\ref{fig:sufficiency}a) displays $\Delta_{\mathrm{sgn}}$ as a function of network depth and width across all three architectures (FFNN, ResNet, ViT; SiLU activation, CIFAR-10). See Appendix \ref{appx_traindetail} for a description of the exact architectures and the general training setup. 
For shallow or narrow networks the sign gap is substantial: the sign code loses information that the full network retains.
As depth and width increase, $\Delta_{\mathrm{sgn}}$ shrinks monotonically and becomes negligible ($< 0.03$) once both dimensions are sufficiently large.
Sign-code sufficiency is not an unconditional property of the architecture; it is an \emph{emergent} property of overparameterized networks.  There is a clear phase boundary in the depth--width plane beyond which the last-layer sign code matches the full network's test performance. This is the regime targeted by the certificate in Section~\ref{sec:theory}. Figure~\ref{fig:sufficiency}b) plots sign-$k$NN against full model test accuracy for overparameterized networks (depth 6, width 1024 FFNN, standard ResNet-18 and depth 6, 256 embedding dimension, 8 heads ViT). We train 540 models spanning three architectures, four sign-preserving activations, three datasets (MNIST, CIFAR-10, CIFAR-100), 5 dataset sizes and 3 seeds. In the overparameterized regime, points cluster tightly around the diagonal with no systematic dependence on architecture, activation, or dataset. The experiment is conservative: it uses only the last layer's sign code, a small fraction of the full routing record, yet this already matches model performance across all four activations, including smooth ones (SiLU, GELU, tanh) with no piecewise-affine structure (thus generalizes the observations in \citep{he2021neural}). This is not a consequence of neural collapse: our models are far from the collapsed regime (NC1 $> 1$, NCM accuracy well below model accuracy, see \cref{appx_NC}). 
\begin{figure}[htb!]
  \centering
  \includegraphics[width=\textwidth]{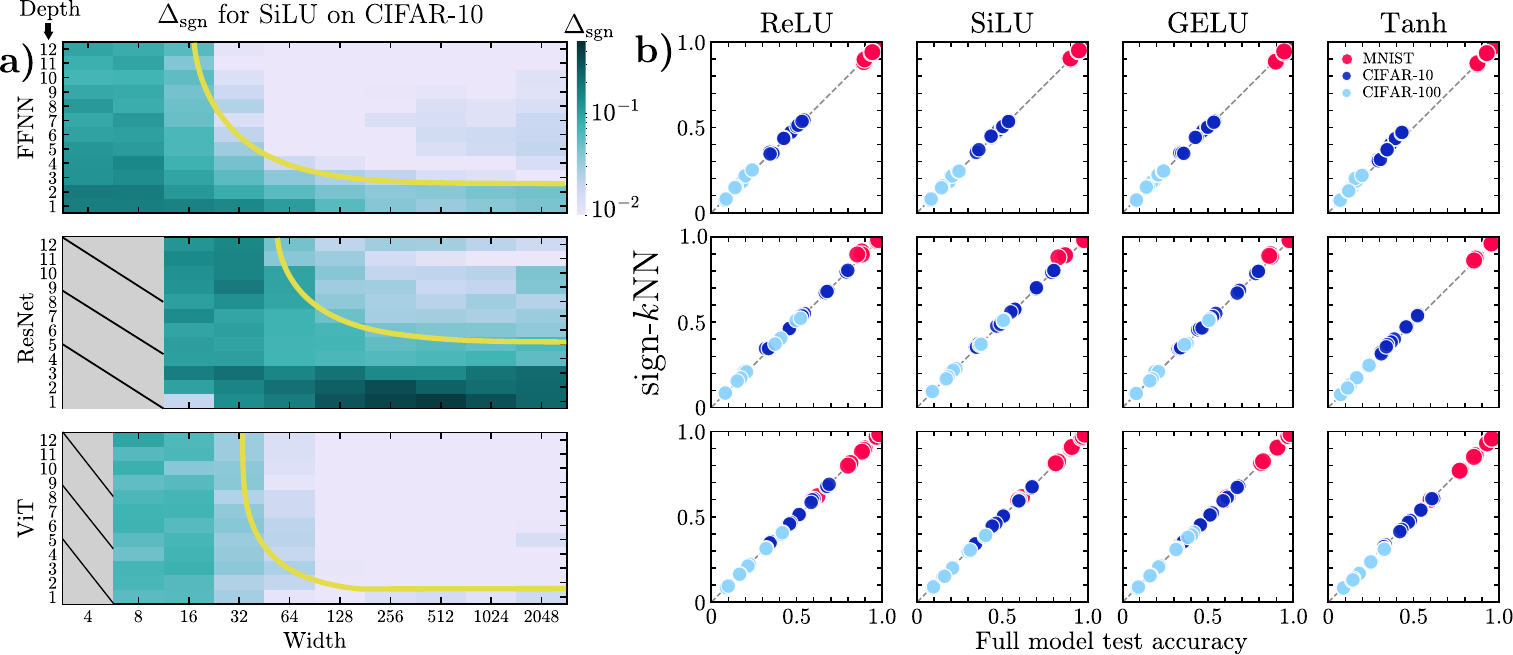}
  \caption{\textbf{a)}~The sign gap $\Delta_{\mathrm{sgn}}$ (Eq.~\ref{eq:sign-gap}) as a function of network depth and width for three architectures (FFNN, ResNet, ViT) with  SiLU on  CIFAR-10.
  The gap shrinks monotonically with increasing depth and width, with a clear phase boundary (yellow lines) beyond which $\Delta_{\mathrm{sgn}} < 0.03$.
  \textbf{b)}~Sign-$k$NN accuracy versus full model test accuracy across standard architectures (rows: FFNN, ResNet-18, ViT), activations (columns: ReLU, SiLU, GELU, Tanh), and datasets (colours: MNIST, CIFAR-10, CIFAR-100). The dashed line is $y{=}x$. Training and model details for all figures are in Appendix \ref{appx_traindetail}.}
  \label{fig:sufficiency}
\end{figure}

\begin{tcolorbox}[
  colback=custombeige!45, colframe=black,
  boxsep=2pt, left=4pt, right=4pt, top=4pt, bottom=2pt
]
\textbf{Observation 1: Last-layer sign codes become informationally sufficient in the overparameterized regime}. The last-layer sign code empirically suffices to match the model's test error; the sign gap $\Delta_{\mathrm{sgn}}$ vanishes with depth and width.
\end{tcolorbox}


\section{Experiment: Testing robustness to sign code and magnitude perturbations}
\label{sec:causal}

The last section showed that sign codes are informative:
a nearest-neighbour classifier on the binary sign code of a single
hidden layer tracks model accuracy.
But informative does not mean causal.
Both sign and magnitude are deterministic functions of the same weights,
so either channel, or both, could carry the classification signal:

\begin{tcolorbox}[
  colback=custombeige!45, colframe=black,
  boxsep=2pt, left=4pt, right=4pt, top=4pt, bottom=2pt
]
\textbf{Question 2:} How does the network respond to distinct perturbations: one that perturbs only the discrete sign code, and one that perturbs only the continuous magnitudes?
\end{tcolorbox}


If the sign code is merely a by-product and the real computation flows
through magnitudes, then corrupting signs should be tolerable and
corrupting magnitudes catastrophic.
If the sign code is the dominant computational channel, the pattern should be the reverse.

\subsection{Experimental setup: Two channel-isolating interventions}

Recall that at every hidden layer~$\ell$, the representation decomposes
as
$\mathbf{h}^{(\ell)} =
\mathbf{S}^{(\ell)} \odot \mathbf{M}^{(\ell)}$
(Definition~\ref{def:sign-mag}).
Both interventions select neurons independently with probability
$\alpha \in [0,1]$ at every hidden layer, but they are not meant to be
symmetric natural forward passes.
They are channel-isolating tests: one destroys routing while preserving
magnitudes, and the other destroys magnitude information while enforcing
the clean routing trajectory.
The interventions differ in \emph{which} channel is corrupted
(Figure~\ref{fig:causal}, right) and in the projection needed to keep
that corruption isolated.

\paragraph{Intervention~1: destroy routing, preserve magnitudes.}
Each selected neuron has its sign negated while its magnitude is
preserved.
The corrupted representation propagates freely through all downstream
layers, potentially altering their sign codes and magnitudes.

\paragraph{Intervention~2: destroy magnitudes, preserve routing.}
We first cache the clean sign code $\mathbf{S}^{(\ell)}(x)$ at every
layer via an unmodified forward pass.
Each selected neuron then has its magnitude replaced by $|\eta|$,
$\eta \sim \mathcal{N}(0,1)$, a class-agnostic random draw carrying
zero information about the input.
The replacement is \emph{projected}: at every downstream layer, the
cached clean sign code is enforced and only magnitudes are recomputed.
This is a causal-mediation test: it asks what remains of the computation
when the sign trajectory is held fixed, not whether the network could
recover that trajectory after magnitude corruption.
At $\alpha = 1$, every magnitude in every layer is a random number, yet
the sign code is perfectly preserved.

\begin{figure}[t]
  \centering
  \includegraphics[width=\textwidth]{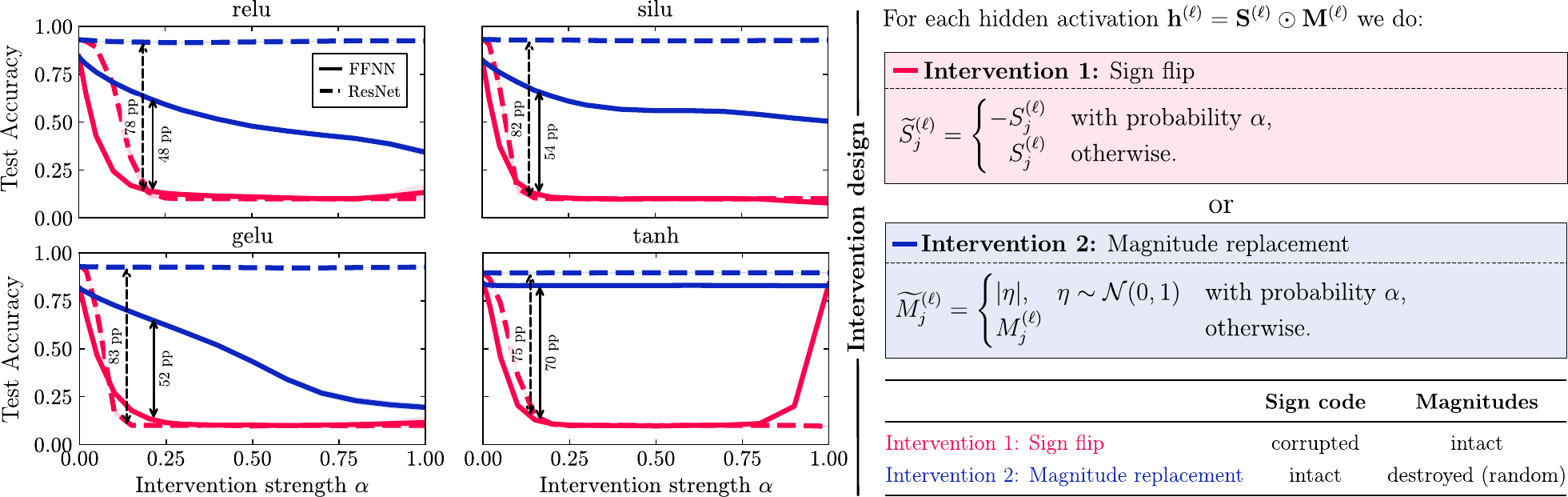}
  \caption{%
    \textbf{Left (2$\times$2):} Accuracy under channel destruction.
    \textcolor{red}{Red}: sign flip (sign code routing destroyed, magnitudes
    intact).
    \textcolor{blue}{Blue}: magnitude replacement with random
    $|\mathcal{N}(0,1)|$ draws (magnitudes destroyed, sign code routing intact).
    Solid: FFNN; dashed: ResNet; mean $\pm 1$\,s.d.\ over 5 draws.
    \textbf{Right:} Intervention formulas and summary of what each
    intervention removes.
    Routing destruction collapses accuracy to chance by
    $\alpha = 0.2$ across all conditions; projected replacement of every
    magnitude preserves far more classification performance, and is nearly
    lossless for residual networks. 
  }
  \label{fig:causal}
\end{figure}


\subsection{Results of the robustness tests}
We sweep $\alpha \in [0,1]$ for four activations (ReLU, SiLU, GELU, tanh) and two architectures (FFNN on SVHN, ResNet-18 on CIFAR-10), averaging over five independent mask draws. Flipping just $20\%$ of sign bits ($\alpha{=}0.2$) collapses accuracy to near chance (${\sim}10\%$) across all eight conditions; the ResNet's skip connections do not compensate (\cref{fig:causal}). At $\alpha{=}1$, the ResNet maintains $93$--$95\%$ accuracy across all four activations: every magnitude in every layer is an independent random number, yet the network classifies as well as with correct magnitudes. The FFNN retains near-full accuracy for tanh and degrades substantially for the rest, but remains far above the corresponding sign-flip failures.

The two interventions are deliberately asymmetric. Intervention~2 asks: \emph{does the network need any magnitude information beyond what the sign code encodes?} Answering this requires isolating the magnitude channel, if corrupted magnitudes cascaded through downstream layers, they would induce secondary sign flips, corrupting both channels simultaneously. Projected replacement caches the clean sign code and enforces it at every layer, holding routing fixed while destroying magnitude content. Intervention~1 needs no analogous projection: negating a sign preserves $|{-}m|{=}|m|$, so the magnitude channel is untouched by construction. The FFNN degradation under Intervention~2 is itself predicted by activation symmetry: tanh is odd, so its output is entirely determined by the sign code (near-zero drop for magnitudes and symmetry upon full sign-code swaps); ReLU, GELU, and SiLU each possess a sign-blind even component $[\varphi(z){+}\varphi(-z)]/2$ invisible to the sign code, with values $0.50$, $0.34$, $0.23$ at $z{=}1$ which matches the degradation ranking (ReLU ${\sim}36$\,pp; GELU ${\sim}34$\,pp; SiLU ${\sim}32$\,pp; tanh ${\sim}2$\,pp). The ResNet's immunity follows from the skip connection preserving sign structure under projected replacement.

\begin{tcolorbox}[
  colback=custombeige!45, colframe=black,
  boxsep=2pt, left=4pt, right=4pt, top=4pt, bottom=2pt
]
\textbf{Observation 2: Networks are robust to magnitude perturbations but not to sign code perturbations.} Projected replacement of every magnitude preserves far more classification performance than sign corruption; flipping $\sim$20\% of signs destroys the forward pass.

\end{tcolorbox}

\section{Experiment: Testing class-coherent sign geometry}
\label{sec:gen}

\cref{sec:sufficient,sec:causal} showed that overparameterized networks can classify through sign code, and that corrupting the sign trajectory is far more destructive than corrupting magnitudes. But computation through signs does not by itself explain generalization. A network could classify through signs and still assign each example an idiosyncratic, non-transferable sign code, hence: 
\begin{tcolorbox}[
  colback=custombeige!45, colframe=black,
  boxsep=2pt, left=4pt, right=4pt, top=4pt, bottom=2pt
]
\textbf{Question 3:} What sign code structure supports generalization?
\end{tcolorbox}

We focus here on evidence for necessity and leave sufficiency to future work.

\subsection{Experimental setup: Varying class neighbourhood purity through a regularizer $\mathcal{L}_{\mathrm{rep}}$ }

\paragraph{Measuring sign code class structure.}
A sign code neighbourhood of a datapoint $\bm{x}_i$ is label-pure if the nearest training
points in Hamming distance share the class of $\bm{x}_i$ .
We measure this property with the leave-one-out (LOO) purity   $\pi_k(\bm{x}_i)$  at $k$ neighbours:
\begin{equation}
  \pi_k(\bm{x}_i)
  = \frac{1}{k}\sum_{j \in N_k^{\setminus i}(\bm{x}_i)}
    \mathbf{1}[y_j = y_i], 
\end{equation}
where $N_k^{\setminus i}$ denotes the $k$ nearest training points in
Hamming distance excluding $\bm{x}_i$ itself.
When $\bar\pi_k = \mathbb{E}_i[\pi_k(\bm{x}_i)]$ is high, same-class inputs
occupy tight, overlapping regions in sign space, exactly the structure
that sign-kNN exploits.
When $\bar\pi_k$ is near $1/C$ (chance), the sign code carries no local
class geometry.
To test the role of class-pure neighbourhoods in generalization, we design a regularizer that actively destroys them. 

\paragraph{Causal test: within-class sign repulsion.} For each neuron we form a continuous proxy of its binary routing
decision, concatenate across layers, and $\ell_2$-normalise.
The repulsion loss $\mathcal{L}_{\mathrm{rep}}$ penalises routing overlap between same-class
pairs:
\begin{equation}\label{eq:repulsion}
  \mathcal{L}_{\mathrm{rep}}
  = \frac{1}{|\mathcal{P}|}
    \sum_{(i,j)\in\mathcal{P}}
    \bigl(\mathbf{S}(\vx_i)^\top   \mathbf{S}(\vx_j) \bigr)^2,
  \qquad
  \mathcal{P}
  = \{(i,j): i \neq j,\; y_i = y_j\}.
\end{equation}
For ReLU, SiLU, and GELU we use a sigmoid gate proxy
$S_j = \sigma(\bm{z}_j/\tau) \in (0,1)$, which targets the active set
that dominates the forward pass; for tanh, where positive and
negative neurons contribute equally, we instead use a tanh sign
proxy $S_j = \tanh(\bm{z}_j/\tau) \in (-1,1)$.
The choice of proxy, and why gate repulsion is strictly stronger
than sign repulsion for non-odd activations, is discussed in
Appendix~\ref{app:reg-gate}.

\subsection{Results}
\begin{figure}[htb!]
  \centering
  \includegraphics[width=\linewidth]{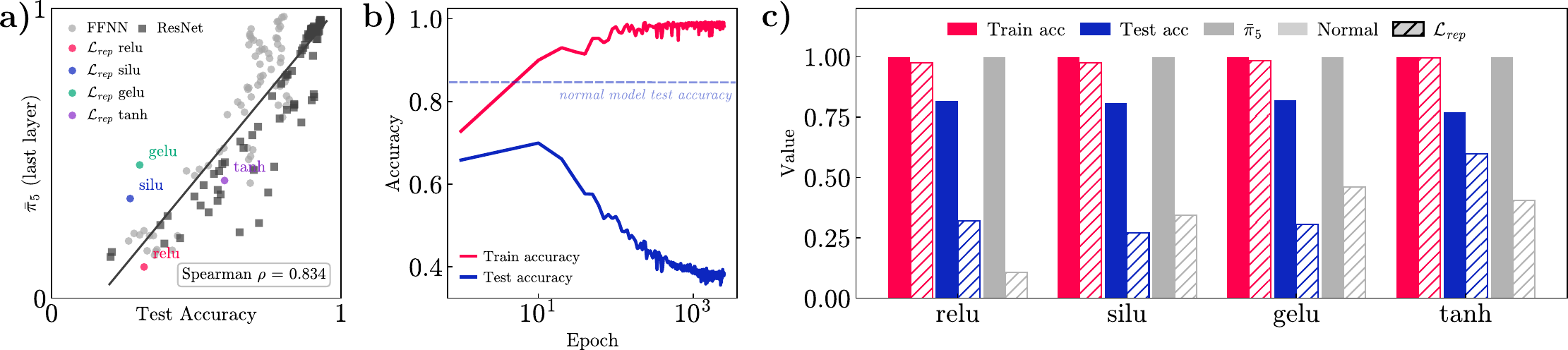}
  \caption{
    \textbf{a)}~Train LOO purity ($k{=}5$) versus test accuracy
    across models (grey: FFNN on SVHN, widths 4--1024, depths 1--12, SiLU) and $\mathcal{L}_{\mathrm{rep}}$-trained models (coloured).
    High neighbourhood purity tracks generalization
    (Spearman $\rho = 0.834$); $\mathcal{L}_{\mathrm{rep}}$
    models fall on the same trend line.
    \textbf{b)}~Training dynamics under
    $\mathcal{L}_{\mathrm{rep}}$ (SiLU, width 1024, depth 8, SVHN).
    Train accuracy reaches $97\%$; test accuracy drops to
    ${\sim}40\%$. Dashed: Test accuracy of the normally trained baseline.
    \textbf{c)}~Summary across activations (depth 8, width 1024,
    SVHN). Repulsion sharply degrades test accuracy and LOO purity for
    ReLU, SiLU, and GELU while training accuracy
    remains~${\geq}97\%$. Tanh is partially robust
    ($\delta \approx 0$; Appendix~\ref{app:asymmetry}).}
  \label{fig:gen}
\end{figure}

Figure~\ref{fig:gen}a) plots train LOO purity against test accuracy for a range of models. Label-pure sign-code neighbourhoods predict good generalization; label-mixed ones the opposite. The $\mathcal{L}_{\mathrm{rep}}$-trained models land in the low-purity, low-accuracy corner \emph{on the same trend line}, confirming the relationship is not an artifact of architecture or dataset. Panel~b) shows SiLU dynamics: train accuracy reaches $97\%$ while test accuracy monotonically drops to ${\sim}40\%$ as the regularizer destroys class-aligned routing. Panel~c) summarises across activations: for ReLU, SiLU, and GELU, repulsion sharply degrades test accuracy and purity while training accuracy stays above $97\%$; tanh is partially robust due to its magnitude symmetry ($\delta \approx 0$, Appendix~\ref{app:asymmetry}). The network still memorises every training example--representational capacity is intact--but the routing patterns are no longer shared across same-class inputs, and generalization collapses. While $\mathcal{L}_{\mathrm{rep}}$ may affect correlated properties, this combination of high training and degraded test accuracy provides causal evidence that within-class sign-code coherence plays a direct role in generalization for the models studied.

\begin{tcolorbox}[
  colback=custombeige!45, colframe=black,
  boxsep=2pt, left=4pt, right=4pt, top=4pt, bottom=2pt
]
\textbf{Observation 3: Class-coherent sign geometry is necessary for generalization.} Networks generalize when same-class inputs share similar sign codes; a regularizer that repels within-class sign codes preserves training accuracy but collapses test accuracy.
\end{tcolorbox}
\section{Theory: Overparameterization Makes Neural Networks Binary}
\label{sec:theory}

\cref{sec:sufficient,sec:causal} showed that sign codes become
informationally sufficient in the overparameterized regime and are
causally dominant under channel-isolating interventions.
We now provide an exact decomposition of the logit into a sign-only
term and a magnitude residual, together with a certificate that predicts
when the sign-only classifier agrees with the full network.

\subsection{Decomposing the logit}
\label{sec:logit-decomp}

The class-$c$ logit decomposes as
$f_c(\vx) = \sum_{j=1}^{d} w_{cj}\,S_j(\vx)\,M_j(\vx) + b_c$
for a last hidden layer of width~$d$.
Replacing each magnitude $M_j$ by its conditional mean
$\mu_j(s) = \mathbb{E}[M_j \mid S_j{=}s]$, the unique MMSE
predictor of $M_j$ from the single bit $S_j$
(Appendix~\ref{app:gauge},
Proposition~\ref{prop:gauge-opt}), splits this into two exact
pieces:
\begin{equation}\label{eq:logit-decomp}
\colorbox{teal!15}{$\displaystyle
  f_c(\vx)
  \;=\;
  \underbrace{
    \sum_j w_{cj}\,S_j(\vx)\,\mu_j\!\bigl(S_j(\vx)\bigr) + b_c
  }_{\displaystyle\tilde f_c(\vx)\;
    \text{: sign-only logit}}
  \;\;+\;\;
  \underbrace{
    \sum_j w_{cj}\,S_j(\vx)\,
    \bigl[\,M_j(\vx) - \mu_j\!\bigl(S_j(\vx)\bigr)\,\bigr]
  }_{\displaystyle R_c(\vx)\;
    \text{: magnitude residual}}
$}
\end{equation}
The sign-only logit $\tilde f_c$ depends on $\vx$ only through the
binary sign code.
The residual $R_c$ captures everything the sign code misses:
the per-neuron deviations of each magnitude from its
sign-conditional average.
The pairwise margin inherits the same split:
\begin{equation}\label{eq:margin-decomp}
  \underbrace{M_{a,c}(x)}_{\text{full margin}}
  \;=\;
  \underbrace{\widetilde{M}_{a,c}(x)}_{\text{sign-only margin}}
  \;+\;
  \underbrace{R_{a,c}(x)}_{\text{residual}}\,,
\end{equation}
and the sign-only prediction disagrees with the full network only
when the residual overwhelms the sign-only margin:
$|R_{a,c}| \ge |\widetilde{M}_{a,c}|$.

\subsection{A certificate for sign domination}
\label{sec:z1}

The decomposition raises a quantitative question:
when is the sign-only margin large enough that the residual cannot
flip the prediction?
For each input $\vx$ with full-model predicted class
$\hat y = \hat f(\vx)$ and each competitor $c \neq \hat y$,
define the per-neuron vote bias
$
\label{eq:vote-bias}
  m_c(\vx)
  \;=\;
  \frac{1}{d}\Bigl[
    \sum_{j=1}^{d}
    (w_{\hat y,j} - w_{c,j})\,S_j(\vx)\,\mu_j\!\bigl(S_j(\vx)\bigr)
    \;+\;
    (b_{\hat y} - b_c)
  \Bigr],$
which measures how decisively, on average, each neuron's
sign-weighted contribution favours $\hat y$ over~$c$, and the
conditional residual moment $
\label{eq:resid-moment}
  r_c^2(\vx)
  \;=\;
  \mathbb{E}\!\bigl[
    R_{\hat y(\vx),c}^{\,2}(X)
    \;\big|\;
    \mathbf{S}(X){=}\mathbf{S}(\vx),\;
    \hat f(X){=}\hat y(\vx)
  \bigr],
$
which captures how much the residual fluctuates among inputs that
share the same sign code and full-model prediction.
Conditioning on both ensures that the class pair and sign-only
margin are fixed under the conditional measure.
This is a population quantity; in the experiments we estimate the same
signal using the residual diagnostics described in
Appendix~\ref{app:theorem-conditions}.
The \textbf{gate-domination score} is:
\begin{equation}\label{eq:z1}
  Z_1(\vx)
  \;=\;
  \min_{c \neq \hat y}\;
  \frac{|\widetilde{M}_{\hat y,c}(\vx)|}{r_c(\vx)}\,.
\end{equation}

\begin{theorem}[Gate domination]
\label{thm:gate-dom}
Let\/ $\varepsilon_1 = \Pr_x(\hat f(x) \neq \hat{\tilde f}(x))$
be the disagreement rate between the full network and its sign-only
version.
Assume that ties in both classifiers occur with probability zero
under the data distribution, and that all residual second moments are
finite.

\smallskip
\noindent\textbf{(i)~Disagreement bound: }
$\label{eq:eps1}
  \varepsilon_1
  \;\le\;
  (C{-}1)\;\mathbb{E}_x\!\bigl[Z_1^{-2}(\vx)\bigr].$
  
\noindent\textbf{(ii)~Width scaling:}
If\/ $r_c^2(\vx) \le d\,\bar\sigma_c^2(\vx)$
for all competitors~$c$, then
$\label{eq:z1-scaling}
  Z_1(x)
  \;\ge\;
  \Bigl(\min_{c \neq \hat y}
  \frac{|m_c(\vx)|}{\bar\sigma_c(\vx)}\Bigr)\sqrt{d}\,.$
  \end{theorem}

\begin{proof}
 The full proof is in Appendix~\ref{app:gate-dom-proof}.
\end{proof}

\subsection{Interpretation and empirical validation}
Figure~\ref{fig:gate-dom} displays $Z_1$, $\varepsilon_1$, and
the sign-gap $\Delta_{\mathrm{sgn}}$ across 360 trained FFNNs
(12~depths $\times$ 10~widths $\times$ 3~seeds, SiLU, CIFAR-10).
The theorem gives a sufficient explanation for why overparameterized
networks can become effectively binary.
Each neuron casts a sign vote for or against each class;
\emph{depth} composes these votes into a decisive bias
($m_c > 0$), and \emph{width} averages out the magnitude noise
through $\sqrt{d}$-scaling.
The sign-only margin grows as $d \cdot m_c$ while the residual
grows at most as $\sqrt{d} \cdot \bar\sigma_c$, so $Z_1$
improves with added neurons until the sign-only classifier
agrees with the full network.
Neither mechanism suffices alone: at depth~1, votes are unbiased
($m_c \approx 0$) and $Z_1 < 0.23$ even at width~4096
(panel~\textbf{b)}, bottom row); narrow-but-deep networks have biased votes
but too few neurons for the averaging to suppress magnitude noise.
The yellow line in panels~(\textbf{a)}--\textbf{c)}) marks where both conditions
are jointly met: above it,
$\Delta_{\mathrm{sgn}} < 0.03$ and $\varepsilon_1 < 0.1$.
Panel~\textbf{d)} is the key result: the plotted configurations collapse
onto a single $Z_1$-vs-$\varepsilon_1$ curve, regardless of
whether a model achieves its $Z_1$ through width, depth, or both.
We verify all theorem conditions and discuss their scope in
Appendix~\ref{app:theorem-conditions}; the analysis is most
informative for FFNNs with asymmetric activations,
while for tanh and residual architectures the causal interventions
(\cref{sec:causal}) provide the stronger evidence.

\begin{figure}[htb!]
\centering
\includegraphics[width=\textwidth]{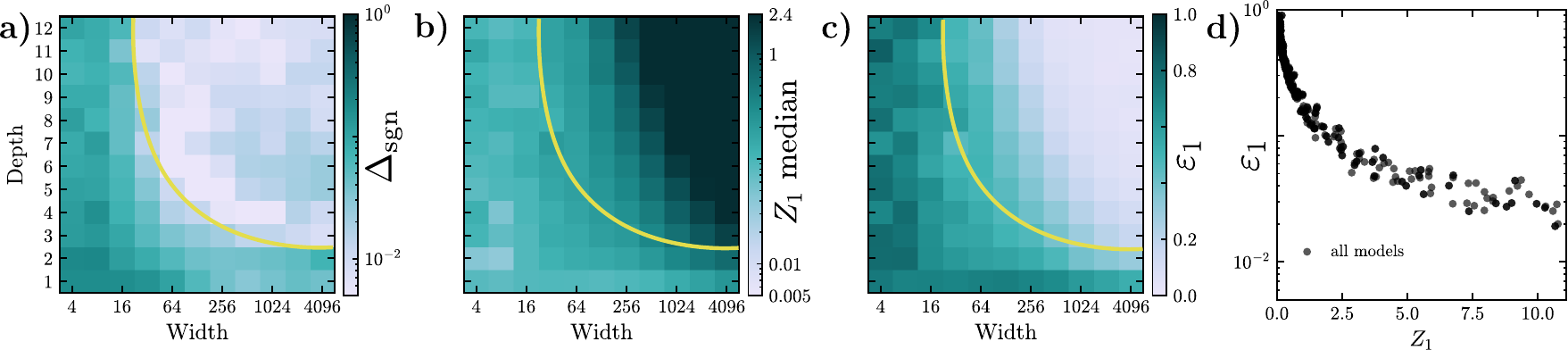}
\caption{%
(FFNN, SiLU, CIFAR-10; widths 4--4096, depths 1--12, mean over
3~seeds).
The yellow line marks the empirical gate-domination region where both the sign gap and sign-only disagreement are small; above it,
$\Delta_{\mathrm{sgn}} < 0.03$ and $\varepsilon_1 < 0.1$
\textbf{(a)}~Sign code kNN gap $\Delta_{\mathrm{sgn}}$: the sign
code matches model test error (gap~${<}0.03$) throughout this
region.
\textbf{(b)}~Median $Z_1$: increases with both width
($\sqrt{d}$-averaging) and depth (stronger vote bias $m_c$).
\textbf{(c)}~Disagreement rate $\varepsilon_1$: mirrors $Z_1$
with the same diagonal phase boundary.
\textbf{(d)}~$Z_1$ predicts $\varepsilon_1$: the disagreement
rate decreases monotonically with $Z_1$ across the plotted
configurations, independent of how each model achieves its $Z_1$
value.}
\label{fig:gate-dom}
\end{figure}

\section{Discussion}
\label{sec:discussion}


In overparameterized networks, classification depends far more on a discrete sign code with class-coherent sign geometry than on precise activation magnitudes. Preserving sign trajectories while randomizing magnitudes retains high accuracy, whereas small sign perturbations collapse performance catastrophically.
This suggests that existing generalization theory may be missing an important granularity: the  binary sign code quotients weight space into equivalence classes, and generalization bounds formulated directly over this discrete object would be invariant to magnitude redundancies that can make norm-based bounds vacuous. Our analysis also offers a hypothesis on why gated activations (SiLU, GELU, and GLU variants) are dominantly used in modern transformers: their large magnitude asymmetry (\cref{app:asymmetry}) amplifies the sign-only margin, making the binary routing channel maximally decisive, whereas symmetric activations like tanh place no signal in this channel.
\paragraph{Limitations.} While we empirically treat multiple architectures, our theoretical analysis centres on feedforward networks. The gate-domination score $Z_1$ can remain modest for residual networks and odd activations even when the sign trajectory is causally dominant under intervention (\cref{app:asymmetry}), and we restrict attention to sign-preserving activations.
Our empirical evidence comes from small-scale vision benchmarks (CIFAR, SVHN, MNIST) and classification tasks; whether the sign code mechanism extends to large-scale settings, language modelling, or regression is left to future work.


\newpage

\begin{ack}
This work was supported by the Laboratory for AI Security Research (LASR). The views expressed in this paper are those of the authors and do not necessarily reflect the position of LASR or His Majesty’s Government.
\end{ack}

\bibliographystyle{plainnat} 
\bibliography{bib}

@article{he2021neural,
  title={Neural networks behave as hash encoders: An empirical study},
  author={He, Fengxiang and Lei, Shiye and Ji, Jianmin and Tao, Dacheng},
  journal={arXiv preprint arXiv:2101.05490},
  year={2021}
}

@inproceedings{park2021unsupervised,
  title={Unsupervised representation learning via neural activation coding},
  author={Park, Yookoon and Lee, Sangho and Kim, Gunhee and Blei, David},
  booktitle={International Conference on Machine Learning},
  pages={8391--8400},
  year={2021},
  organization={PMLR}
}

@inproceedings{geng2023towards,
  title={Towards reliable neural specifications},
  author={Geng, Chuqin and Le, Nham and Xu, Xiaojie and Wang, Zhaoyue and Gurfinkel, Arie and Si, Xujie},
  booktitle={International Conference on Machine Learning},
  pages={11196--11212},
  year={2023},
  organization={PMLR}
}

@inproceedings{zhang2017rethinking,
  title={Understanding deep learning requires rethinking generalization},
  author={Zhang, Chiyuan and Bengio, Samy and Hardt, Moritz and Recht, Benjamin and Vinyals, Oriol},
  booktitle={International Conference on Learning Representations},
  year={2017}
}

@article{dziugaite2017computing,
  title={Computing nonvacuous generalization bounds for deep (stochastic) neural networks with many more parameters than training data},
  author={Dziugaite, Gintare Karolina and Roy, Daniel M},
  journal={arXiv preprint arXiv:1703.11008},
  year={2017}
}

@inproceedings{neyshabur2015norm,
  title={Norm-based capacity control in neural networks},
  author={Neyshabur, Behnam and Tomioka, Ryota and Srebro, Nathan},
  booktitle={Conference on learning theory},
  pages={1376--1401},
  year={2015},
  organization={PMLR}
}

@article{bartlett2017spectrally,
  title={Spectrally-normalized margin bounds for neural networks},
  author={Bartlett, Peter L and Foster, Dylan J and Telgarsky, Matus J},
  journal={Advances in neural information processing systems},
  volume={30},
  year={2017}
}

@inproceedings{arora2018stronger,
  title={Stronger generalization bounds for deep nets via a compression approach},
  author={Arora, Sanjeev and Ge, Rong and Neyshabur, Behnam and Zhang, Yi},
  booktitle={International conference on machine learning},
  pages={254--263},
  year={2018},
  organization={PMLR}
}

@inproceedings{
  jiang2020fantastic,
  title={Fantastic Generalization Measures and Where to Find Them},
  author={Yiding Jiang and Behnam Neyshabur and Hossein Mobahi and Dilip Krishnan and Samy Bengio},
  booktitle={International Conference on Learning Representations},
  year={2020},
  url={https://openreview.net/forum?id=SJgIPJBFvH}
}

@article{belkin2019reconciling,
  title={Reconciling modern machine-learning practice and the classical bias--variance trade-off},
  author={Belkin, Mikhail and Hsu, Daniel and Ma, Siyuan and Mandal, Soumik},
  journal={Proceedings of the National Academy of Sciences},
  volume={116},
  number={32},
  pages={15849--15854},
  year={2019},
  publisher={National Academy of Sciences}
}

@article{nakkiran2021deep,
  title={Deep double descent: Where bigger models and more data hurt},
  author={Nakkiran, Preetum and Kaplun, Gal and Bansal, Yamini and Yang, Tristan and Barak, Boaz and Sutskever, Ilya},
  journal={Journal of Statistical Mechanics: Theory and Experiment},
  volume={2021},
  number={12},
  pages={124003},
  year={2021},
  publisher={IOP Publishing and SISSA}
}

@article{neyshabur2018role,
  title={Towards Understanding the Role of Over-Parametrization in Generalization of Neural Networks},
  author={Behnam Neyshabur and Zhiyuan Li and Srinadh Bhojanapalli and Yann LeCun and Nathan Srebro},
  journal={ArXiv},
  year={2018},
  volume={abs/1805.12076},
  url={https://api.semanticscholar.org/CorpusID:44130076}
}

@article{soudry2018implicit,
author = {Soudry, Daniel and Hoffer, Elad and Nacson, Mor Shpigel and Gunasekar, Suriya and Srebro, Nathan},
title = {The implicit bias of gradient descent on separable data},
year = {2018},
issue_date = {January 2018},
publisher = {JMLR.org},
volume = {19},
number = {1},
issn = {1532-4435},
journal = {J. Mach. Learn. Res.},
month = jan,
pages = {2822–2878},
numpages = {57}
}

@article{lyu2020gradient,
  title={Gradient Descent Maximizes the Margin of Homogeneous Neural Networks},
  author={Kaifeng Lyu and Jian Li},
  journal={ArXiv},
  year={2019},
  volume={abs/1906.05890},
  url={https://api.semanticscholar.org/CorpusID:189898036}
}

@inproceedings{chizat2020implicit,
  title={Implicit Bias of Gradient Descent for Wide Two-layer Neural Networks Trained with the Logistic Loss},
  author={L{\'e}na{\"i}c Chizat and Francis R. Bach},
  booktitle={Annual Conference Computational Learning Theory},
  year={2020},
  url={https://api.semanticscholar.org/CorpusID:211076120}
}

@article{keskar2017largebatch,
  title={On Large-Batch Training for Deep Learning: Generalization Gap and Sharp Minima},
  author={Nitish Shirish Keskar and Dheevatsa Mudigere and Jorge Nocedal and Mikhail Smelyanskiy and Ping Tak Peter Tang},
  journal={ArXiv},
  year={2016},
  volume={abs/1609.04836},
  url={https://api.semanticscholar.org/CorpusID:5834589}
}

@inproceedings{dinh2017sharp,
  title={Sharp Minima Can Generalize For Deep Nets},
  author={Laurent Dinh and Razvan Pascanu and Samy Bengio and Yoshua Bengio},
  booktitle={International Conference on Machine Learning},
  year={2017},
  url={https://api.semanticscholar.org/CorpusID:7636159}
}

@inproceedings{glorot2011deep,
  title={Deep Sparse Rectifier Neural Networks},
  author={Xavier Glorot and Antoine Bordes and Yoshua Bengio},
  booktitle={International Conference on Artificial Intelligence and Statistics},
  year={2011},
  url={https://api.semanticscholar.org/CorpusID:2239473}
}

@article{pascanu2014number,
  title={On the number of response regions of deep feed forward networks with piece-wise linear activations},
  author={Razvan Pascanu and Guido Mont{\'u}far and Yoshua Bengio},
  journal={arXiv: Learning},
  year={2013},
  url={https://api.semanticscholar.org/CorpusID:9254343}
}

@inproceedings{montufar2014number,
  title={On the Number of Linear Regions of Deep Neural Networks},
  author={Guido Mont{\'u}far and Razvan Pascanu and Kyunghyun Cho and Yoshua Bengio},
  booktitle={Neural Information Processing Systems},
  year={2014},
  url={https://api.semanticscholar.org/CorpusID:5941770}
}

@inproceedings{raghu2017expressive,
  title={On the Expressive Power of Deep Neural Networks},
  author={Maithra Raghu and Ben Poole and Jon M. Kleinberg and Surya Ganguli and Jascha Narain Sohl-Dickstein},
  booktitle={International Conference on Machine Learning},
  year={2016},
  url={https://api.semanticscholar.org/CorpusID:2838204}
}

@article{serra2018bounding,
  title={Bounding and Counting Linear Regions of Deep Neural Networks},
  author={Thiago Serra and Christian Tjandraatmadja and Srikumar Ramalingam},
  journal={ArXiv},
  year={2017},
  volume={abs/1711.02114},
  url={https://api.semanticscholar.org/CorpusID:34019680}
}

@inproceedings{hanin2019complexity,
  title={Complexity of Linear Regions in Deep Networks},
  author={Boris Hanin and David Rolnick},
  booktitle={International Conference on Machine Learning},
  year={2019},
  url={https://api.semanticscholar.org/CorpusID:59291990}
}

@inproceedings{gamba2022regions,
  title={Are All Linear Regions Created Equal?},
  author={Matteo Gamba and Adrian Chmielewski-Anders and Josephine Sullivan and Hossein Azizpour and Marten Bjorkman},
  booktitle={International Conference on Artificial Intelligence and Statistics},
  year={2022},
  url={https://api.semanticscholar.org/CorpusID:247084088}
}

@inproceedings{balestriero2018spline,
  title={A Spline Theory of Deep Networks},
  author={Randall Balestriero and Richard Baraniuk},
  booktitle={International Conference on Machine Learning},
  year={2018},
  url={https://api.semanticscholar.org/CorpusID:51882762}
}

@article{neyshabur2015pathsgd,
  title={Path-SGD: Path-Normalized Optimization in Deep Neural Networks},
  author={Behnam Neyshabur and Ruslan Salakhutdinov and Nathan Srebro},
  journal={ArXiv},
  year={2015},
  volume={abs/1506.02617},
  url={https://api.semanticscholar.org/CorpusID:2101905}
}

@article{lakshminarayanan2020neural,
  title={Neural Path Features and Neural Path Kernel : Understanding the role of gates in deep learning},
  author={Chandrashekar Lakshminarayanan and Amit Singh},
  journal={ArXiv},
  year={2020},
  volume={abs/2006.10529},
  url={https://api.semanticscholar.org/CorpusID:219792921}
}

@article{fiat2019decoupling,
  title={Decoupling Gating from Linearity},
  author={Jonathan Fiat and Eran Malach and Shai Shalev-Shwartz},
  journal={ArXiv},
  year={2018},
  volume={abs/1906.05032},
  url={https://api.semanticscholar.org/CorpusID:127096892}
}

@article{gonon2024pathnorm,
  title={A path-norm toolkit for modern networks: consequences, promises and challenges},
  author={Antoine Gonon and Nicolas Brisebarre and Elisa Riccietti and R{\'e}mi Gribonval},
  journal={ArXiv},
  year={2023},
  volume={abs/2310.01225},
  url={https://api.semanticscholar.org/CorpusID:263605441}
}

@article{lebeurrier2026pathconditioned,
  title={Path-conditioned training: a principled way to rescale ReLU neural networks},
  author={Arthur Lebeurrier and Titouan Vayer and R{\'e}mi Gribonval},
  journal={ArXiv},
  year={2026},
  volume={abs/2602.19799},
  url={https://api.semanticscholar.org/CorpusID:285973448}
}

@inproceedings{huang2020visualizations,
  title={Understanding Generalization through Visualizations},
  author={W. Ronny Huang and Zeyad Ali Sami Emam and Micah Goldblum and Liam H. Fowl and Justin K. Terry and Furong Huang and Tom Goldstein},
  booktitle={ICBINB@NeurIPS},
  year={2019},
  url={https://api.semanticscholar.org/CorpusID:182952973}
}

@article{kornblith2019similarity,
  title={Similarity of Neural Network Representations Revisited},
  author={Simon Kornblith and Mohammad Norouzi and Honglak Lee and Geoffrey E. Hinton},
  journal={ArXiv},
  year={2019},
  volume={abs/1905.00414},
  url={https://api.semanticscholar.org/CorpusID:141460329}
}

@article{papyan2020prevalence,
  title={Prevalence of neural collapse during the terminal phase of deep learning training},
  author={Vardan Papyan and Xuemei Han and David L. Donoho},
  journal={Proceedings of the National Academy of Sciences of the United States of America},
  year={2020},
  volume={117},
  pages={24652 - 24663},
  url={https://api.semanticscholar.org/CorpusID:221172897}
}

@article{han2022neural,
  title={Neural Collapse Under MSE Loss: Proximity to and Dynamics on the Central Path},
  author={X. Y. Han and Vardan Papyan and David L. Donoho},
  journal={ArXiv},
  year={2021},
  volume={abs/2106.02073},
  url={https://api.semanticscholar.org/CorpusID:235352583}
}

@article{he2016deep,
  title={Deep Residual Learning for Image Recognition},
  author={Kaiming He and X. Zhang and Shaoqing Ren and Jian Sun},
  journal={2016 IEEE Conference on Computer Vision and Pattern Recognition (CVPR)},
  year={2015},
  pages={770-778},
  url={https://api.semanticscholar.org/CorpusID:206594692}
}

@article{dosovitskiy2021image,
  title={An Image is Worth 16x16 Words: Transformers for Image Recognition at Scale},
  author={Alexey Dosovitskiy and Lucas Beyer and Alexander Kolesnikov and Dirk Weissenborn and Xiaohua Zhai and Thomas Unterthiner and Mostafa Dehghani and Matthias Minderer and Georg Heigold and Sylvain Gelly and Jakob Uszkoreit and Neil Houlsby},
  journal={ArXiv},
  year={2020},
  volume={abs/2010.11929},
  url={https://api.semanticscholar.org/CorpusID:225039882}
}

@article{hendrycks2016gelu,
  title={Gaussian Error Linear Units (GELUs)},
  author={Dan Hendrycks and Kevin Gimpel},
  journal={arXiv: Learning},
  year={2016},
  url={https://api.semanticscholar.org/CorpusID:125617073}
}

@article{ramachandran2018searching,
  title={Searching for Activation Functions},
  author={Prajit Ramachandran and Barret Zoph and Quoc V. Le},
  journal={ArXiv},
  year={2018},
  volume={abs/1710.05941},
  url={https://api.semanticscholar.org/CorpusID:10919244}
}

@article{elfwing2018silu,
title = {Sigmoid-weighted linear units for neural network function approximation in reinforcement learning},
journal = {Neural Networks},
volume = {107},
pages = {3-11},
year = {2018},
note = {Special issue on deep reinforcement learning},
issn = {0893-6080},
doi = {https://doi.org/10.1016/j.neunet.2017.12.012},
url = {https://www.sciencedirect.com/science/article/pii/S0893608017302976},
author = {Stefan Elfwing and Eiji Uchibe and Kenji Doya}
}

@article{han2016deepcompression,
  title={Deep Compression: Compressing Deep Neural Network with Pruning, Trained Quantization and Huffman Coding},
  author={Song Han and Huizi Mao and William J. Dally},
  journal={arXiv: Computer Vision and Pattern Recognition},
  year={2015},
  url={https://api.semanticscholar.org/CorpusID:2134321}
}

@article{frankle2019lottery,
  title={The Lottery Ticket Hypothesis: Finding Sparse, Trainable Neural Networks},
  author={Jonathan Frankle and Michael Carbin},
  journal={arXiv: Learning},
  year={2018},
  url={https://api.semanticscholar.org/CorpusID:53388625}
}

\newpage


\appendix


\section{Related work}
\label{appx:related}



A central motivation for this paper is the generalization puzzle in overparameterized deep learning.
Modern networks can interpolate random labels while still generalizing on natural data, forcing a re-examination of classical capacity arguments \citep{zhang2017rethinking}.
This led to a large literature on norm-, margin-, PAC-Bayes-, and compression-based explanations of generalization \citep{dziugaite2017computing,neyshabur2015norm,bartlett2017spectrally,arora2018stronger,jiang2020fantastic}, and to studies of the interpolation regime through double descent and overparameterization \citep{belkin2019reconciling,nakkiran2021deep,neyshabur2018role}.
A complementary line analyzes the implicit bias of optimization, showing convergence toward max-margin or otherwise structured solutions in linear, homogeneous, and wide-network settings \citep{soudry2018implicit,lyu2020gradient,chizat2020implicit}.
Another influential debate ties generalization to sharpness or flatness of minima \citep{keskar2017largebatch,dinh2017sharp}.
Our work is motivated by the same puzzle, but shifts the explanatory target from continuous weights to the discrete routing pattern induced by the network.

A second line studies neural networks through their activation patterns and the partitions they induce.
Early work on rectifiers emphasized sparse activation and the practical advantages of non-saturating units \citep{glorot2011deep}.
For ReLU and related piecewise-linear networks, several papers quantified the number of response or linear regions and the expressive gains from depth \citep{pascanu2014number,montufar2014number,raghu2017expressive,serra2018bounding}.
Later work showed that the number of linear regions realized in practice can be far below worst-case counts \citep{hanin2019complexity}, and that linear regions are not equally important functionally \citep{gamba2022regions}.
Closely related spline and max-affine views interpret deep networks as input-dependent affine maps indexed by activation patterns \citep{balestriero2018spline}.
Our sign code perspective is closest to this tradition, but differs in emphasis: the goal here is not primarily to count regions or characterize expressivity, but to argue that the binary code identifying a region can itself become the dominant carrier of class information.

Some aspects of binary activation patterns have been studied. It was shown in \cite{he2021neural} that the sign code induced by a trained classifier can behave as a hash representation. Park \textit{et al.}  introduces a specialized unsupervised objective that learns binary activation codes in order to increase encoder expressivity and improve downstream classification and retrieval \cite{park2021unsupervised}.
Related work has shown that class-dependent partial sign codes can be extracted from trained ReLU networks and used to verify that the network assigns the same label to inputs consistent with those partial sign codes \cite{geng2023towards}.

This paper is especially close to the path- and gate-based literature.
Path-SGD and path-norm analyses emphasized geometries and complexity measures that respect ReLU rescaling symmetries better than raw Euclidean weight space \citep{neyshabur2015pathsgd,neyshabur2015norm}.
Lakshminarayanan and Singh introduced neural path features and the neural path kernel, explicitly encoding the active subnetwork selected by the gates \citep{lakshminarayanan2020neural}.
Fiat et al.\ made the separation between gating and linear computation explicit by decoupling the two roles in the neuron \citep{fiat2019decoupling}.
More recent work extended path-based analyses to modern DAG architectures and principled rescaling procedures \citep{gonon2024pathnorm,lebeurrier2026pathconditioned}.
Among existing lines, this is the closest antecedent to our thesis: we likewise elevate gates from a bookkeeping device to a primary object of analysis, but we go further by arguing for three complementary claims---informational sufficiency, causal dominance under intervention, and theoretical dominance of the sign code once magnitude residuals concentrate.

Our empirical results also connect to work on representation geometry in trained classifiers.
Visualization and representation-similarity studies show that networks that generalize organize their internal representations differently from networks that memorize \citep{huang2020visualizations,kornblith2019similarity}.
Neural collapse sharpened this picture by showing that late-layer features can converge toward class means and nearest-class-mean decision rules during terminal training phases \citep{papyan2020prevalence,han2022neural}.
Our sign-$k$NN and codebook-collapse analyses are complementary to this literature: they replace continuous Euclidean geometry by discrete Hamming geometry on sign codes, and ask whether class structure is already recoverable before or even without full collapse of the continuous feature cloud.

Finally, because the present paper studies sign-preserving activations and multiple backbone families, it also interfaces with the activation-function and architecture literature.
The experiments span residual networks and vision transformers \citep{he2016deep,dosovitskiy2021image}, and include GELU and SiLU/Swish-style smooth activations in addition to ReLU and tanh \citep{hendrycks2016gelu,ramachandran2018searching,elfwing2018silu}.
The redundancy and erasure-robustness story we observe is also adjacent to pruning, compression, and sparse-subnetwork results \citep{han2016deepcompression,frankle2019lottery}: heavily overparameterized networks often contain substantial slack.
The distinction is that pruning and compression typically operate in parameter space, whereas our thesis is about input-dependent routing structure.
The claim is not merely that many weights are removable, but that overparameterization can induce a simpler, more redundant binary partition whose class geometry increasingly governs prediction.

\section{Additional material for the main paper}
\subsection{Extension of sign extraction}
\label{appx_signextr}
\paragraph{Convolutional blocks.}
For a convolutional output
$\mathbf{h}^{(\ell)} \in \mathbb{R}^{C \times H \times W}$,
Definition~\ref{def:sign-mag} applies entrywise, yielding a
sign code tensor
$\mathbf{S}^{(\ell)} \in \{-1,+1\}^{C \times H \times W}$ that
records both \emph{which} channels and \emph{where} in space each
feature responds positively.
We compress the spatial dimensions by majority vote on a
$P{\times}P$ grid ($P{=}2$ throughout), giving $C \cdot P^2$ bits
per layer.

\paragraph{Residual blocks.}
A residual block adds a learned branch to a skip connection:
\begin{equation}
  \mathbf{h}^{(\ell+1)}
  = \varphi\!\Bigl(
      \mathbf{h}^{(\ell)}
      + \mathcal{F}^{(\ell)}\!\bigl(\mathbf{h}^{(\ell)}\bigr)
    \Bigr)\,,
  \label{eq:resblock}
\end{equation}
where $\mathcal{F}^{(\ell)}$ is a sub-network (typically two
convolutional or fully-connected layers with their own
activations).
Two families of sign codes arise.
The activations \emph{inside} the branch
$\mathcal{F}^{(\ell)}$ produce \emph{internal} sign codes,
computed exactly as in Definition~\ref{def:sign-mag}.
For a standard two-layer branch with intermediate activation, the
internal code is
$\mathbf{S}_{\mathrm{int}}^{(\ell)} =
\operatorname{sgn}\!\bigl(
\varphi(W_1^{(\ell)}\mathbf{h}^{(\ell)} +
\mathbf{b}_1^{(\ell)})\bigr)$.
The \emph{block-output} code
$\mathbf{S}_{\mathrm{out}}^{(\ell)} =
\operatorname{sgn}(\mathbf{h}^{(\ell+1)})$
records the sign of the \emph{sum} of skip and branch, and need
not equal $\mathbf{S}^{(\ell)}$: it is precisely the sign flips
introduced by the branch that allow each block to refine the
combinatorial partition of the preceding layer. The full sign code of an $L$-block residual network concatenates
both internal and output codes across all blocks:
\begin{equation}
  \mathbf{S}(x)
  \;=\;
  \bigl(\,
    \mathbf{S}_{\mathrm{int}}^{(1)},\;
    \mathbf{S}_{\mathrm{out}}^{(1)},\;
    \ldots,\;
    \mathbf{S}_{\mathrm{int}}^{(L)},\;
    \mathbf{S}_{\mathrm{out}}^{(L)}
  \,\bigr)
  \;\in\; \{-1,+1\}^{D}\,.
  \label{eq:resnet-full-code}
\end{equation}
Omitting either the internal or the output codes would lose
information about the routing decisions the network makes.

\subsection{Relationship to Neural Collapse}
\label{appx_NC}

Neural collapse \citep{papyan2020prevalence,han2022neural} describes a terminal-phase phenomenon in which (NC1) within-class feature variance vanishes relative to between-class variance, (NC2) class means converge to an equiangular tight frame, (NC3) the last-layer classifier aligns with the class means, and (NC4) a nearest-class-mean (NCM) classifier matches network accuracy.
A natural question is whether sign code sufficiency is simply a consequence of neural collapse: if features have already collapsed onto class means, any reasonable representation—including the sign code—would trivially separate classes.

We measure all four NC metrics across the full FFNN width–depth grid (SiLU, SVHN, 80 training epochs) in \cref{fig:nc_diagnostics}.
The models are far from the collapsed regime.
NC1 ($S_W/S_B$) remains above $1$ throughout most of the grid, indicating that within-class scatter still dominates between-class scatter; values below $10^{-1}$ would indicate meaningful collapse.
NC3 (cosine alignment between the classifier weights and the class-mean directions) stays below $0.4$ for most configurations, far from the ideal value of $1$.
NC4 (NCM test accuracy) peaks around $60$--$70\%$, substantially below the full model accuracy of ${\sim}80\%$ achieved by the same networks.
NC2 (ETF deviation) is near zero at large widths but this alone does not indicate collapse—it reflects only the angular structure of the class means, not the tightness of the feature clouds around them.

Despite this absence of collapse, the sign code is already informationally sufficient (\cref{sec:sufficient}) and causally dominant under intervention (\cref{sec:causal}) in precisely the same models.
Sign purity—the fraction of training points whose sign code is shared only with same-class examples—is high across much of the grid even where NC1 remains large and NC4 is poor (compare bottom-left and bottom-centre panels of \cref{fig:nc_diagnostics}).
Class structure is thus recoverable in discrete Hamming space well before the continuous features converge toward their class means.
This makes the sign code mechanism strictly more general than neural collapse: it does not require prolonged terminal-phase training, and it operates through combinatorial routing geometry rather than through concentration of continuous features.

\begin{figure}[htb!]
  \centering
  \includegraphics[width=\textwidth]{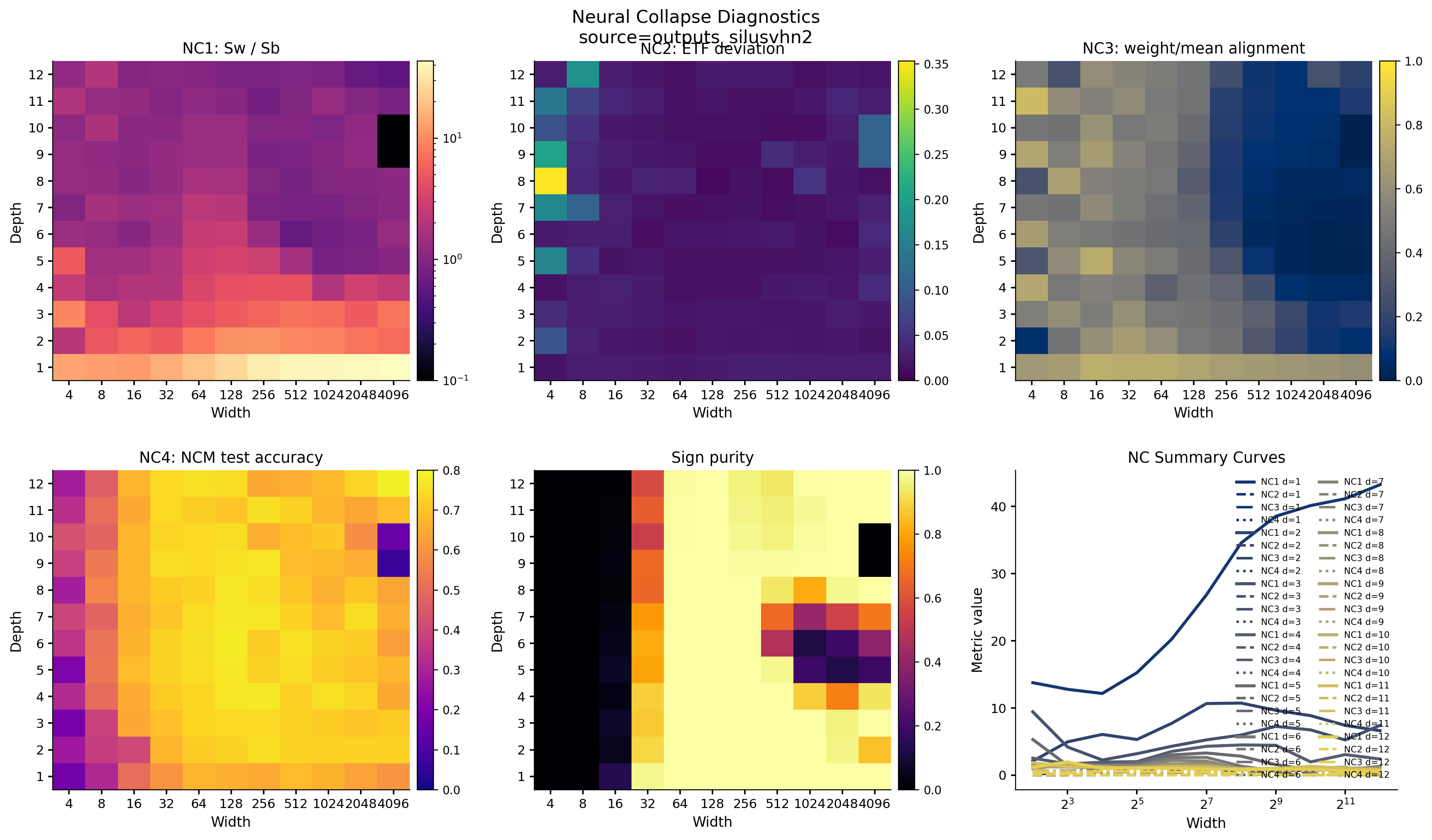}
  \caption{\textbf{Neural collapse diagnostics across the width–depth grid} (FFNN, SiLU, SVHN, 80 epochs).
  \textbf{Top row:} NC1 ($S_W/S_B$, log scale), NC2 (ETF deviation), NC3 (weight–mean cosine alignment).
  \textbf{Bottom row:} NC4 (nearest-class-mean test accuracy), sign purity, and NC summary curves across widths for selected depths.
  NC1 remains above $1$ and NC3 below $0.4$ across most configurations, confirming that the models have not undergone neural collapse.
  NCM test accuracy (NC4) peaks at ${\sim}70\%$, well below full model accuracy.
  Despite this, sign purity is high wherever the network is sufficiently wide and deep—class-discriminative sign code structure emerges independently of, and prior to, neural collapse.}
  \label{fig:nc_diagnostics}
\end{figure}

\section{Non-Sign-Preserving Activations}
\label{app:non-sign-preserving}

The sign-preserving property (Definition~\ref{def:sign-preserving}) covers the activations dominant in modern practice: ReLU, SiLU/Swish, GELU, and tanh are all sign-preserving and account for the vast majority of architectures in current use. The decomposition $h_j = S_j \cdot M_j$ with $S_j = \operatorname{sgn}(z_j)$ applies directly to all of them. We briefly discuss what happens outside this class.

\paragraph{Shifted-threshold activations (e.g.\ softplus).}
Softplus, $\varphi(z) = \log(1 + e^z)$, is strictly positive for all $z$ and therefore not sign-preserving: $\operatorname{sgn}(\varphi(z)) = +1$ always. However, softplus is monotone, so it preserves the ordering of pre-activations. A natural generalization is to replace the zero threshold with the activation's symmetry axis. For softplus, $\varphi(0) = \log 2$, and the function is approximately linear for $z \gg 0$ and approximately flat for $z \ll 0$. Defining $S_j = \operatorname{sgn}(z_j)$ directly—bypassing the activation output—recovers a meaningful binary partition: it still records which side of the neuron's learned hyperplane the input falls on. The magnitude can then be defined relative to this threshold, e.g.\ $M_j = |\varphi(z_j) - \varphi(0)|$. We have not tested this variant experimentally, but the monotonicity of softplus ensures that the partition induced by $\operatorname{sgn}(z_j)$ is identical to the one studied throughout the paper—only the magnitude profile changes.

\paragraph{Symmetric activations (e.g.\ $z^2$, $\cos(z)$).}
Squaring or other even functions destroy sign information entirely: $(-z)^2 = z^2$. The sign--magnitude decomposition as stated does not apply. Yet a binary routing question remains meaningful—one could ask whether $|z_j|$ exceeds some data-dependent threshold $\tau_j$ (e.g.\ the median pre-activation magnitude), partitioning inputs into high-activation and low-activation regimes. Whether such a partition carries class-discriminative structure comparable to the sign code is an open empirical question. We note that even activations like $z^2$ are rarely used in isolation in modern networks; they typically appear inside gated units (e.g.\ $z_1 \cdot z_2^2$ in some GLU variants) where the gating component restores a sign-preserving partition on at least one branch.

In summary, the sign-preserving assumption is not restrictive in practice: it covers all standard activations, and for the main exception (softplus), the pre-activation sign provides an equivalent partition. Extending the framework to genuinely non-monotone or even activations is possible in principle but would require a different choice of binary code, and we leave this to future work.

\section{Sign Code Extraction}
\label{app:extraction}

This appendix documents the precise extraction procedure for the
sign code $\mathbf{S}(x)$ used throughout the paper.
We describe the procedure separately for each architecture family
and specify exactly which tensors are recorded, which are
discarded, and how the full-network sign code is assembled.
All sign codes use the convention
$\operatorname{sgn}(t) = +1$ if $t > 0$ and $-1$ otherwise (in
particular, $\operatorname{sgn}(0) = -1$).

\subsection{Feedforward networks (FFNN)}
\label{app:extraction-ffnn}

\paragraph{Architecture.}
Our FFNN models are of the form
\[
  x
  \;\xrightarrow{\text{flatten}}\;
  \mathbf{h}^{(0)}
  \;\xrightarrow{W^{(1)},\,\varphi}\;
  \mathbf{h}^{(1)}
  \;\xrightarrow{W^{(2)},\,\varphi}\;
  \cdots
  \;\xrightarrow{W^{(L)},\,\varphi}\;
  \mathbf{h}^{(L)}
  \;\xrightarrow{W^{(L+1)}}\;
  \text{logits}\,,
\]
where every hidden layer applies the same sign-preserving
activation $\varphi$ (one of ReLU, SiLU, GELU, or tanh).
The model exposes an \texttt{nn.ModuleList} of
\texttt{nn.Linear} layers (\texttt{hidden\_layers}) and a final
\texttt{nn.Linear} output layer (\texttt{output\_layer}).

\paragraph{Extraction.}
We register a forward hook on each \texttt{nn.Linear} in
\texttt{hidden\_layers}.
The hook records the \emph{pre-activation}
$\mathbf{z}^{(\ell)} = W^{(\ell)}\mathbf{h}^{(\ell-1)} +
\mathbf{b}^{(\ell)}$, i.e.\ the output of the linear layer
\emph{before} the activation is applied.
After the forward pass completes, we apply $\varphi$ to each
captured pre-activation to obtain the post-activation
$\mathbf{h}^{(\ell)} = \varphi(\mathbf{z}^{(\ell)})$, and
compute:
\begin{align}
  \mathbf{S}^{(\ell)}(x)
  &= \operatorname{sgn}\!\bigl(\mathbf{h}^{(\ell)}(x)\bigr)
  \;\in\; \{-1,+1\}^{d_\ell}\,,
  \label{eq:app-ffnn-sign} \\
  \mathbf{M}^{(\ell)}(x)
  &= \bigl|\mathbf{h}^{(\ell)}(x)\bigr|
  \;\in\; \mathbb{R}_{\geq 0}^{d_\ell}\,.
  \label{eq:app-ffnn-mag}
\end{align}
Because $\varphi$ is sign-preserving,
$\mathbf{S}^{(\ell)} = \operatorname{sgn}(\mathbf{z}^{(\ell)})$,
so the sign code is equivalently determined by the pre-activation
signs.

\paragraph{Full sign code.}
The full-network sign code is the concatenation across all $L$
hidden layers:
\[
  \mathbf{S}(x)
  = \bigl(\mathbf{S}^{(1)}(x),\;\ldots,\;\mathbf{S}^{(L)}(x)\bigr)
  \;\in\; \{-1,+1\}^D\,,
  \qquad
  D = \textstyle\sum_{\ell=1}^{L} d_\ell\,.
\]
The output layer $W^{(L+1)}$ is not activated and is therefore
excluded from the sign code.

\paragraph{What is \emph{not} extracted.}
The output logits are recorded for evaluation but do not
contribute to the sign code.
No sign code is computed for the input $\mathbf{h}^{(0)} = x$
(there is no activation to decompose).

\subsection{Residual networks (ResNet)}
\label{app:extraction-resnet}

\paragraph{Architecture.}
We use torchvision \texttt{BasicBlock}-based ResNets (ResNet-18,
ResNet-34).
The architecture comprises a convolutional stem followed by four
layer groups, each containing multiple residual blocks, followed
by global average pooling and a linear classifier:
\[
  x
  \;\xrightarrow[\text{stem}]{\text{conv1, bn1, relu, maxpool}}\;
  \mathbf{h}_{\mathrm{stem}}
  \;\xrightarrow{\text{layer1}}\;
  \cdots
  \;\xrightarrow{\text{layer4}}\;
  \mathbf{h}_{\mathrm{final}}
  \;\xrightarrow{\text{avgpool, fc}}\;
  \text{logits}\,.
\]
Each \texttt{BasicBlock} has the following internal structure:
\[
  \mathbf{h}^{(\ell)}
  \;\xrightarrow{\text{conv1, bn1}}\;
  \mathbf{z}_{\mathrm{int}}
  \;\xrightarrow{\text{relu}}\;
  \mathbf{a}
  \;\xrightarrow{\text{conv2, bn2}}\;
  \mathbf{b}
  \;\xrightarrow{+\;\mathrm{skip}}\;
  \mathbf{z}_{\mathrm{out}}
  \;\xrightarrow{\text{relu}}\;
  \mathbf{h}^{(\ell+1)}\,,
\]
where the skip connection is the identity when dimensions match,
or a learned $1{\times}1$ convolution with batch normalisation
(\texttt{downsample}) when a spatial or channel-dimension change
occurs.
Note that torchvision \texttt{BasicBlock} uses a \emph{single}
\texttt{nn.ReLU} module that is called twice during the forward
pass: once after \texttt{bn1} (internal) and once after the
residual addition (output).

\paragraph{Extraction.}
We register five hooks per block and one hook on the stem:

\begin{enumerate}[leftmargin=1.5em, itemsep=2pt]
\item \textbf{Stem activation.}\;
  A forward hook on \texttt{model.relu} captures
  $\mathbf{h}_{\mathrm{stem}} =
  \mathrm{ReLU}(\mathrm{bn1}(\mathrm{conv1}(x)))$, the stem
  output after the first ReLU (but before max-pooling).
  We record
  $\mathbf{S}_{\mathrm{stem}} =
  \operatorname{sgn}(\mathbf{h}_{\mathrm{stem}})$ and
  $\mathbf{M}_{\mathrm{stem}} =
  |\mathbf{h}_{\mathrm{stem}}|$.

\item \textbf{Block input.}\;
  A forward \emph{pre}-hook on the block module captures the
  input tensor $\mathbf{h}^{(\ell)}$ entering the block.
  This serves as the skip connection when no downsampling is
  needed.

\item \textbf{Internal branch activation.}\;
  A forward hook on \texttt{block.relu} captures every ReLU
  output inside the block.
  In a \texttt{BasicBlock} the shared ReLU is called exactly
  twice; the \emph{first} output is the internal branch
  activation
  $\mathbf{a} =
  \mathrm{ReLU}(\mathrm{bn1}(\mathrm{conv1}(\mathbf{h}^{(\ell)})))$.
  We record
  $\mathbf{S}_{\mathrm{int}}^{(\ell)} =
  \operatorname{sgn}(\mathbf{a})$ and
  $\mathbf{M}_{\mathrm{int}}^{(\ell)} = |\mathbf{a}|$.

\item \textbf{Branch output (pre-addition).}\;
  A forward hook on \texttt{block.bn2} captures
  $\mathbf{b} = \mathrm{bn2}(\mathrm{conv2}(\mathbf{a}))$, the
  output of the residual branch before addition with the skip.

\item \textbf{Skip connection.}\;
  If the block has a \texttt{downsample} module, a forward hook
  captures its output.
  Otherwise, the skip is the block input $\mathbf{h}^{(\ell)}$
  recorded in step~2.

\item \textbf{Residual sum and block output.}\;
  We compute the residual sum
  $\mathbf{z}_{\mathrm{out}} = \mathbf{b} + \mathrm{skip}$
  from the tensors captured in steps 4 and~5.
  This is the pre-activation of the block's output ReLU.
  The block-output activation
  $\mathbf{h}^{(\ell+1)} =
  \mathrm{ReLU}(\mathbf{z}_{\mathrm{out}})$
  is captured via a forward hook on the block module itself.
  We record
  $\mathbf{S}_{\mathrm{out}}^{(\ell)} =
  \operatorname{sgn}(\mathbf{h}^{(\ell+1)})$ and
  $\mathbf{M}_{\mathrm{out}}^{(\ell)} =
  |\mathbf{h}^{(\ell+1)}|$.
  Because ReLU is sign-preserving, $\mathbf{S}_{\mathrm{out}}^{(\ell)}
  = \operatorname{sgn}(\mathbf{z}_{\mathrm{out}})$; we
  additionally store $\operatorname{sgn}(\mathbf{z}_{\mathrm{out}})$
  and $|\mathbf{z}_{\mathrm{out}}|$ for diagnostic use.
\end{enumerate}

\paragraph{Full sign code.}
The full-network sign code concatenates the stem code and both the
internal and output codes from every block:
\[
  \mathbf{S}(x)
  = \bigl(\,
    \mathbf{S}_{\mathrm{stem}},\;
    \mathbf{S}_{\mathrm{int}}^{(1)},\;
    \mathbf{S}_{\mathrm{out}}^{(1)},\;
    \ldots,\;
    \mathbf{S}_{\mathrm{int}}^{(B)},\;
    \mathbf{S}_{\mathrm{out}}^{(B)}
  \,\bigr)\,,
\]
where $B$ is the total number of \texttt{BasicBlock}s across all
four layer groups (e.g.\ $B=8$ for ResNet-18).
For convolutional layers the sign code tensors have shape
$C \times H \times W$; we compress each to $C \cdot P^2$ bits
using majority-vote spatial pooling on a $P{\times}P$ grid
($P=2$).

\paragraph{What is \emph{not} extracted.}
The global average-pooling layer and the final linear classifier
(\texttt{fc}) are not activated by a sign-preserving nonlinearity
and do not contribute to the sign code.
The second ReLU output per block (captured by the shared
\texttt{block.relu} hook) is functionally identical to the
block-output activation captured by the block-level hook; we use
the latter and discard the former.
The stem's max-pooling layer is applied \emph{after} the stem
ReLU and does not introduce a new sign decision, so no separate
code is extracted for it.

\subsection{Vision Transformer (ViT)}
\label{app:extraction-vit}

\paragraph{Architecture.}
Our ViT follows the standard pre-norm design:
\[
  x
  \;\xrightarrow{\text{patch embed}}\;
  \mathbf{t}
  \;\xrightarrow{+\;\text{cls, pos}}\;
  \mathbf{h}^{(0)}
  \;\xrightarrow{\text{block}_1}\;
  \cdots
  \;\xrightarrow{\text{block}_B}\;
  \mathbf{h}^{(B)}
  \;\xrightarrow{\text{norm, head}}\;
  \text{logits}\,,
\]
where each transformer block has the structure
\begin{align*}
  \mathbf{h}_{\mathrm{attn}}^{(\ell)}
  &= \mathbf{h}^{(\ell-1)}
     + \mathrm{Attn}\!\bigl(
       \mathrm{LN}_1(\mathbf{h}^{(\ell-1)})
     \bigr)\,, \\
  \mathbf{h}^{(\ell)}
  &= \mathbf{h}_{\mathrm{attn}}^{(\ell)}
     + \mathrm{MLP}\!\bigl(
       \mathrm{LN}_2(\mathbf{h}_{\mathrm{attn}}^{(\ell)})
     \bigr)\,.
\end{align*}
The MLP sub-block is
$\mathrm{Linear} \to \mathrm{GELU} \to \mathrm{Dropout} \to
\mathrm{Linear} \to \mathrm{Dropout}$.

\paragraph{Extraction.}
Unlike the FFNN and ResNet extractors, the ViT extractor performs
a \emph{manual forward pass} through the model's modules rather
than relying on hooks.
This gives explicit access to every intermediate tensor.
For each transformer block $\ell = 1,\ldots,B$, we record three
sign--magnitude pairs:

\begin{enumerate}[leftmargin=1.5em, itemsep=2pt]
\item \textbf{Post-attention residual stream.}\;
  $\mathbf{h}_{\mathrm{attn}}^{(\ell)} =
  \mathbf{h}^{(\ell-1)} + \mathrm{Attn}(\mathrm{LN}_1(
  \mathbf{h}^{(\ell-1)}))$.
  We record
  $\mathbf{S}_{\mathrm{attn}}^{(\ell)} =
  \operatorname{sgn}(\mathbf{h}_{\mathrm{attn}}^{(\ell)})$
  and
  $\mathbf{M}_{\mathrm{attn}}^{(\ell)} =
  |\mathbf{h}_{\mathrm{attn}}^{(\ell)}|$.
  This is the sign of the residual stream \emph{after} the
  attention sub-block has added its contribution.

\item \textbf{MLP hidden activation.}\;
  $\mathbf{a}^{(\ell)} =
  \mathrm{GELU}\!\bigl(
  W_1^{(\ell)}\,\mathrm{LN}_2(
  \mathbf{h}_{\mathrm{attn}}^{(\ell)}) + \mathbf{b}_1^{(\ell)}
  \bigr)$.
  We record
  $\mathbf{S}_{\mathrm{mlp}}^{(\ell)} =
  \operatorname{sgn}(\mathbf{a}^{(\ell)})$
  and
  $\mathbf{M}_{\mathrm{mlp}}^{(\ell)} =
  |\mathbf{a}^{(\ell)}|$.
  This is the internal sign code of the MLP, analogous to the
  internal branch code in a ResNet block.
  GELU is sign-preserving, so
  $\mathbf{S}_{\mathrm{mlp}}^{(\ell)}$ is determined by the
  pre-activation signs.

\item \textbf{Block output.}\;
  $\mathbf{h}^{(\ell)} =
  \mathbf{h}_{\mathrm{attn}}^{(\ell)} +
  \mathrm{MLP}(\mathrm{LN}_2(
  \mathbf{h}_{\mathrm{attn}}^{(\ell)}))$.
  We record
  $\mathbf{S}_{\mathrm{out}}^{(\ell)} =
  \operatorname{sgn}(\mathbf{h}^{(\ell)})$
  and
  $\mathbf{M}_{\mathrm{out}}^{(\ell)} =
  |\mathbf{h}^{(\ell)}|$.
  As with ResNets, the output sign code need not equal the
  input sign code: sign flips introduced by the MLP branch
  refine the partition.
\end{enumerate}

\noindent
Additionally, we record the patch-embedding tokens
$\mathbf{t} = \mathrm{PatchEmbed}(x)$ and their sign and
magnitude.
These serve as the ``stem'' of the ViT and are available for
diagnostic use, though they do not pass through a sign-preserving
activation and are therefore not included in the default sign
code.

\paragraph{Full sign code.}
The full-network sign code concatenates all three code families
across all $B$ transformer blocks:
\[
  \mathbf{S}(x)
  = \bigl(\,
    \mathbf{S}_{\mathrm{attn}}^{(1)},\;
    \mathbf{S}_{\mathrm{mlp}}^{(1)},\;
    \mathbf{S}_{\mathrm{out}}^{(1)},\;
    \ldots,\;
    \mathbf{S}_{\mathrm{attn}}^{(B)},\;
    \mathbf{S}_{\mathrm{mlp}}^{(B)},\;
    \mathbf{S}_{\mathrm{out}}^{(B)}
  \,\bigr)\,.
\]
Each sign code tensor has shape
$(\text{batch},\; T,\; d_{\mathrm{model}})$ for the residual-stream
codes and
$(\text{batch},\; T,\; d_{\mathrm{mlp}})$ for the MLP hidden
code, where $T$ is the sequence length (number of patches plus
one \texttt{[CLS]} token).
For the experiments in the main paper we flatten over the token
dimension.

\paragraph{What is \emph{not} extracted.}
The self-attention weights $A \in \Delta_T^T$ are continuous
routing variables and do not admit a binary sign decomposition;
they are not included in the sign code.
The final layer norm and the classification head do not involve a
sign-preserving activation and are excluded.

\subsection{Spatial compression of convolutional sign codes}
\label{app:spatial-compression}

Convolutional sign code tensors
$\mathbf{S}^{(\ell)} \in \{-1,+1\}^{C \times H \times W}$ can
be very high-dimensional (e.g.\
$256 \times 8 \times 8 = 16{,}384$ bits for a single layer of
ResNet-18 on $32 \times 32$ inputs).
To obtain a fixed-length code that is comparable across spatial
resolutions, we partition each channel's $H \times W$ feature map
into a $P \times P$ grid of (approximately) equal patches and
replace each patch by its majority sign:
\[
  \bar{S}^{(\ell)}_{c,p,q}
  = \operatorname{sgn}\!\Bigl(
    \textstyle\sum_{(h,w) \in \mathrm{patch}(p,q)}
    S^{(\ell)}_{c,h,w}
  \Bigr)\,.
\]
This yields $C \cdot P^2$ bits per layer.
We use $P = 2$ throughout, which balances spatial resolution
against code dimensionality.
The majority-vote rule means that $\bar{S}^{(\ell)}_{c,p,q} = +1$
iff more than half of the spatial locations in the patch have a
positive activation, preserving the dominant sign in each spatial
neighbourhood.

\subsection{Random projection for high-dimensional codes}
\label{app:random-projection}

Even after spatial compression, the full-network sign code can
have $D \sim 10^4$--$10^5$ dimensions.
Computing exact Hamming distances between all pairs of $n$ points
in $\{-1,+1\}^D$ costs $\mathcal{O}(n^2 D)$, which becomes
prohibitive for large $n$ and $D$.
We reduce the dimensionality by projecting the sign code with a
random binary matrix
$R \in \{-1,+1\}^{D \times d}$ (entries i.i.d.\
Rademacher, $d \ll D$) and re-binarising:
\[
  \tilde{\mathbf{S}}(x)
  = \operatorname{sgn}\!\bigl(\mathbf{S}(x)\, R\bigr)
  \;\in\; \{-1,+1\}^d\,.
\]
By standard results on random hyperplane hashing, the normalised
Hamming distance between two projected codes satisfies
\(
  \mathbb{E}\bigl[d_H(\tilde{\mathbf{S}},\,
  \tilde{\mathbf{S}}')\bigr]
  = \tfrac{1}{\pi}\arccos\!\bigl(
  \cos(\theta)\bigr)
\)
where $\theta$ is the angle between the original sign vectors
treated as real-valued vectors, providing a monotone approximation
to the original Hamming distance.
In all experiments labelled ``projected'' in the main paper, we
use $d = 2{,}048$.

\subsection{Summary of extracted quantities}
\label{app:extraction-summary}

Table~\ref{tab:extraction-summary} summarises which tensors are
recorded for each architecture and how they contribute to the
full sign code.

\begin{table}[h]
  \centering
  \caption{%
    Summary of sign code extraction across architectures.
    \emph{Int.} = internal branch/MLP hidden code;
    \emph{Out.} = block-output code;
    \emph{Attn.} = post-attention residual stream code.
    ``Full $\mathbf{S}$'' indicates inclusion in the full-network
    sign code; indicates the code
    used for the last-layer sign-space $k$-NN experiment in
    Section~\ref{sec:sufficient}.
  }
  \label{tab:extraction-summary}
  \smallskip
  \renewcommand{\arraystretch}{1.2}
  \begin{tabular}{@{}lllcc@{}}
    \toprule
    \textbf{Architecture}
      & \textbf{Tensor}
      & \textbf{Code type}
      & \textbf{Full $\mathbf{S}$}
      & \textbf{\S\cref{sec:sufficient}} \\
    \midrule
    \multirow{2}{*}{FFNN}
      & $\mathbf{h}^{(\ell)} = \varphi(\mathbf{z}^{(\ell)})$
      & Layer sign
      & \checkmark
      & last only \\
      & Output logits
      & ---
      & $\times$
      & $\times$ \\
    \midrule
    \multirow{4}{*}{ResNet}
      & Stem ReLU output
      & Stem
      & \checkmark
      & $\times$ \\
      & $\mathbf{a} = \mathrm{ReLU}(\mathrm{bn1}(\mathrm{conv1}(\cdot)))$
      & Int.\
      & \checkmark
      & $\times$ \\
      & $\mathbf{h}^{(\ell+1)} = \mathrm{ReLU}(\mathbf{b} + \mathrm{skip})$
      & Out.\
      & \checkmark
      & last only \\
      & Avgpool / fc
      & ---
      & $\times$
      & $\times$ \\
    \midrule
    \multirow{5}{*}{ViT}
      & Patch tokens (pre-transformer)
      & Stem
      & $\times$
      & $\times$ \\
      & $\mathbf{h}_{\mathrm{attn}}^{(\ell)}$
      & Attn.\
      & \checkmark
      & $\times$ \\
      & $\mathbf{a}^{(\ell)} = \mathrm{GELU}(\cdot)$
      & Int.\
      & \checkmark
      & $\times$ \\
      & $\mathbf{h}^{(\ell)}$ (block output)
      & Out.\
      & \checkmark
      & last only \\
      & Final norm / head
      & ---
      & $\times$
      & $\times$ \\
    \bottomrule
  \end{tabular}
\end{table}

\paragraph{Which codes are used where.}
The informational sufficiency experiment
(\cref{sec:sufficient}) uses
only the \emph{last} hidden layer's sign code: the final hidden
layer for FFNNs, the output code of the last
\texttt{BasicBlock} for ResNets, and the output code of the last
transformer block for ViTs.
This is a deliberately conservative choice: a single layer's
binary code---a small fraction of the full-network sign
code---already accounts for the network's generalization
performance.
The full-network sign code (all layers concatenated) is used in
later sections where tighter certificates or mechanistic
decompositions require the complete routing record.

\section{Training Details}
\label{appx_traindetail}
\subsection{Training Details \& Hyperparameters}
\label{app:training}

\paragraph{Scope.}
This appendix summarizes the training protocol used for the main
correlation sweeps, the intervention experiments, and the
repulsion-regularized runs.

\paragraph{Baseline sweeps.}
The main width--depth heatmaps use CIFAR-10 with the full
$50{,}000$-example training set. The cross-dataset scatter in
Figure~\ref{fig:sufficiency}b uses three architecture families
(FFNN, ResNet-18, ViT), four activations (ReLU, SiLU, GELU, tanh),
three datasets (MNIST, CIFAR-10, CIFAR-100), five training-set sizes
($1$k, $5$k, $10$k, $25$k, $50$k), and three random seeds, for a
total of $540$ trained models.

\paragraph{Optimization.}
\begin{itemize}
\item \textbf{Loss:} cross-entropy throughout; label smoothing is
  zero by default, with ViT on CIFAR-10 using $0.1$.
\item \textbf{Optimizers:} AdamW for FFNN and ViT; SGD with Nesterov
  momentum for ResNet-18.
\item \textbf{Base learning rates:} FFNN uses base learning rate
  $3\times 10^{-4}$; ResNet-18 uses $10^{-1}$; ViT uses
  $1.5\times10^{-4}$ on CIFAR-10.
\item \textbf{Scaling rule:} for FFNN sweeps, the effective learning
  rate is
  \[
  \mathrm{lr}_{\mathrm{eff}}
  = 3\times10^{-4}\cdot
  \mathrm{clip}\!\left(
  \left(\frac{1024}{w}\right)^{1/2}
  \left(\frac{6}{d}\right)^{1/2},
  0.1,1.0\right),
  \]
  with an additional factor $0.2$ for tanh. For ViT sweeps, the
  effective learning rate is
  \[
  \mathrm{lr}_{\mathrm{eff}}
  = 1.5\times10^{-4}\cdot
  \mathrm{clip}\!\left(
  \left(\frac{1024}{w}\right)
  \left(\frac{6}{d}\right)^{1/2},
  0.125,1.0\right).
  \]
\item \textbf{Schedules:} FFNN uses a constant learning rate in the
  baseline sweeps. ResNet-18 and ViT use cosine decay; ResNet-18 uses
  5 warmup epochs and ViT uses 10 warmup epochs.
\item \textbf{Batch sizes:} 256 for FFNN, 64 for ResNet-18, and 128
  for ViT.
\item \textbf{Weight decay:} $10^{-4}$ for FFNN, $5\times10^{-4}$
  for ResNet-18, and $5\times10^{-2}$ for ViT.
\item \textbf{Mixed precision:} enabled for all large sweeps unless
  explicitly disabled.
\end{itemize}

\paragraph{Stopping rule.}
The main heatmap sweeps use a maximum of 120 epochs together with a
practical early-stop rule based on training accuracy:
$0.95$ for FFNN, $0.995$ for ResNet-18, and $0.98$ for ViT.
These thresholds were used as a compute-saving heuristic once the
model had clearly fit the training set; they were not used in the
repulsion runs, where late-training dynamics are part of the result.

\paragraph{Data and augmentation.}
\begin{itemize}
\item \textbf{Splits:} standard train/test splits; no held-out
  validation split.
\item \textbf{Augmentation:} ResNet-18 and ViT use
  RandomCrop$(32,\mathrm{padding}=4)$ and RandomHorizontalFlip on
  CIFAR; horizontal flips are disabled for SVHN. FFNN uses no data
  augmentation.
\item \textbf{Normalization:} standard per-dataset mean/std
  normalization.
\end{itemize}

\paragraph{Initialization and normalization.}
FFNN layers use Kaiming initialization for ReLU/SiLU/GELU and Xavier
initialization for tanh. ResNet-18 uses the standard convolutional
and batch-normalized residual-block parameterization. ViT uses the
standard pre-norm transformer parameterization with layer norm.

\paragraph{Seeds and counts.}
The fixed-size cross-dataset sweep uses three seeds per configuration.
The main width--depth heatmaps use one seed per grid point. The
CIFAR-10 heatmaps in Figure~\ref{fig:sufficiency}a comprise
132 FFNN models, 108 ResNet-18 models, and 113 ViT models
(353 total).

\paragraph{Intervention experiments.}
The sign-flip and magnitude-replacement experiments use corruption
rates
\[
0,\ 0.01,\ 0.02,\ 0.05,\ 0.1,\ 0.15,\ 0.2,\ 0.25,\ 0.3,\ 0.4,\ 0.5,\
0.6,\ 0.7,\ 0.8,\ 0.9,\ 1.0,
\]
with five independent corruption masks per setting, applied to all
hidden layers.

\paragraph{Repulsion runs.}
The repulsion experiments are trained on SVHN with longer schedules
(maximum 2500 epochs), AdamW, the same FFNN learning-rate rule, and
no early stopping.
All other baseline hyperparameters are kept fixed unless explicitly
overridden by the regularizer.

\paragraph{Compute budget.}
The full project required approximately \textbf{400 H100 GPU-hours}
on a 4$\times$H100 cluster, including the main sweeps, intervention
experiments, repulsion reruns, and auxiliary failed/restarted jobs.
This number should be understood as a rough end-to-end estimate
rather than an exact accounting.

\subsection{Architecture Specifications}
\label{app:arch}

\paragraph{FFNN.}
Inputs are flattened to vectors. All hidden layers share the same
width $w$, and the depth $d$ equals the number of hidden layers.
The activation is ReLU, SiLU, GELU, or tanh, followed by a linear
classifier. There is no final normalization layer before the output
head.

\paragraph{ResNet-18.}
We use a CIFAR-style residual network rather than the ImageNet
torchvision stem: the input stem is a $3\times 3$ convolution with
stride 1 and no max-pooling. Residual blocks use batch normalization,
and the residual branch is scaled by $\beta/\sqrt{B}$ with
$\beta=1$ and $B$ the total number of residual blocks. Channel widths
follow $(1,2,4,8)\times(\text{base width})$, scaled uniformly by the
chosen width multiplier.

\paragraph{Depth allocation for ResNet-18.}
When the total residual depth is varied, blocks are distributed across
the four stages by round-robin allocation. Concretely, depth
$d=5$ gives stage layout $(2,1,1,1)$, depth $d=6$ gives
$(2,2,1,1)$, and depth $d=10$ gives $(3,3,2,2)$.

\paragraph{ViT.}
The ViT models use patch size 4, a learned class token, learned
positional embeddings, pre-norm transformer blocks, multi-head
self-attention with 8 heads, an MLP ratio of 4, and a linear
classifier on the \texttt{[CLS]} token. Width controls the embedding
dimension and depth controls the number of transformer blocks.

\paragraph{Sweep grids.}
\begin{itemize}
\item \textbf{FFNN (Figure~\ref{fig:sufficiency}a):}
  widths $\{4,8,16,32,64,128,256,512,1024,2048,4096\}$ and
  depths $\{1,\dots,12\}$.
\item \textbf{ResNet-18 (Figure~\ref{fig:sufficiency}a):}
  widths $\{1,2,4,8,16,32,64,128,256\}$ and depths $\{1,\dots,12\}$.
\item \textbf{ViT (Figure~\ref{fig:sufficiency}a):}
  widths $\{8,16,32,64,128,256,512,1024,2048,4096\}$ and
  depths $\{1,\dots,12\}$, with the seven largest models
  $(4096,6\text{--}12)$ omitted.
\item \textbf{Purity/repulsion grid (Figure~\ref{fig:gen}):}
  FFNN widths $\{4,16,64,256,512,1024\}$ and
  depths $\{1,2,4,6,8,10\}$; ResNet-18 widths
  $\{1,4,16,32,64,128\}$ and depths $\{1,3,5,7,8,9,10\}$.
\end{itemize}

\paragraph{Reference parameter counts.}
For CIFAR-sized inputs, representative models have:
\begin{itemize}
\item \textbf{FFNN $w{=}1024,d{=}8$:} 10,504,202 parameters.
\item \textbf{ResNet-18 $w{=}16,d{=}8$:} 701,466 parameters.
\item \textbf{ViT $w{=}256,d{=}6$:} 4,771,082 parameters.
\end{itemize}

\subsection{Repulsion Regularizer}
\label{app:repulsion}

\paragraph{Definition.}
For a batch with labels $y$, we form a continuous proxy for the sign
or gate pattern by concatenating selected pre-activations across
layers, applying a temperature-scaled proxy,
$g=\sigma(z/\tau)$ for ReLU/SiLU/GELU and $g=\tanh(z/\tau)$ for tanh,
and then $\ell_2$-normalizing. The repulsion loss penalizes
within-class alignment:
\[
\mathcal{L}_{\mathrm{rep}} =
\frac{1}{|\mathcal{P}|}\sum_{(i,j)\in\mathcal{P}} (g_i^\top g_j)^2,
\qquad
\mathcal{P}=\{(i,j): i\neq j,\ y_i=y_j\}.
\]
The total loss is
\[
\mathcal{L}=\mathcal{L}_{\mathrm{CE}}+\lambda(t)\,
\mathcal{L}_{\mathrm{rep}}.
\]

\paragraph{Schedule and layers.}
The regularization weight is linearly ramped from zero:
\[
\lambda(t)=\lambda_0\cdot
\min\!\Bigl(1,\max\bigl(0,\tfrac{t-t_0}{T}\bigr)\Bigr).
\]
For the Figure~\ref{fig:gen} runs, $t_0=5$ epochs and $T=200$.
The proxy is computed from the pre-activations of all FFNN hidden
layers.

\paragraph{Training protocol.}
The repulsion runs keep the baseline optimizer and data pipeline, but
use a long schedule (2500 epochs) with no early stopping so that the
late-stage collapse in test accuracy is visible. In the main SVHN
experiments, the regularizer uses $\tau=0.005$ and
$\lambda_0=500$.

\paragraph{Sensitivity.}
The effect is not confined to a single regularization strength: in
additional SVHN reruns, both $\lambda_0=250$ and $\lambda_0=500$
preserved near-perfect training accuracy while producing the same
qualitative collapse in test accuracy and purity.

\begin{algorithm}[t]
\caption{Repulsion-regularized training}
\label{alg:repulsion}
\begin{algorithmic}[1]
\For{each minibatch $(x,y)$}
  \State compute logits and hidden pre-activations $\{z^{(\ell)}\}$
  \State form proxy features $g \gets \mathrm{concat}_{\ell}\,
    \phi(z^{(\ell)}/\tau)$, where $\phi=\sigma$ or $\tanh$
  \State normalize each $g_i \gets g_i / (\|g_i\|_2 + \varepsilon)$
  \State form same-class pairs
    $\mathcal{P}=\{(i,j): i\neq j,\ y_i=y_j\}$
  \State compute
    $\mathcal{L}_{\mathrm{rep}} =
    |\mathcal{P}|^{-1}\sum_{(i,j)\in\mathcal{P}}(g_i^\top g_j)^2$
  \State compute
    $\mathcal{L}=\mathcal{L}_{\mathrm{CE}}+\lambda(t)\mathcal{L}_{\mathrm{rep}}$
  \State backpropagate $\mathcal{L}$ and update parameters
\EndFor
\end{algorithmic}
\end{algorithm}

\paragraph{Regularizer-specific hyperparameters.}
\begin{itemize}
\item \textbf{Weight:} $\lambda_0=500$ in the main runs.
\item \textbf{Temperature:} $\tau=0.005$.
\item \textbf{Normalization constant:} $\varepsilon=10^{-8}$.
\item \textbf{Proxy:} sigmoid for ReLU/SiLU/GELU; tanh for tanh.
\item \textbf{Layers:} all hidden layers.
\end{itemize}

\subsection{Sign-kNN Protocol \& Sensitivity}
\label{app:knn}

\paragraph{Protocol.}
Sign-$k$NN is computed from the last-layer sign code. Codes are stored
as $\{-1,+1\}$ vectors and compared by dot product, which is
equivalent to Hamming distance up to an affine transformation. The
main-paper sweeps use a fixed $k=15$ with majority-vote prediction and
deterministic tie-breaking by class count. Training examples form the
reference bank and test examples are used as queries.

\paragraph{Random projection.}
When projected sign-$k$NN is reported, the code is first mapped by a
random $\{\pm1\}$ projection to dimension $d=512$ and averaged over
five independent projections. The main sign-$k$NN results in the paper
use the full sign code rather than the projected code.

\paragraph{Sensitivity to $k$.}
To check that the main conclusions are not sensitive to the choice of
$k$, we re-evaluated three overparameterized SVHN SiLU FFNN
checkpoints using a 20k-example training bank and 5k test queries.
The table reports \emph{accuracy} (not error rate):

\begin{center}
\begin{tabular}{lcccccc}
\toprule
Model & model acc. & $k{=}1$ & $k{=}3$ & $k{=}5$ & $k{=}10$ & $k{=}20$ \\
\midrule
$w{=}256,d{=}8$  & 0.817 & 0.799 & 0.807 & 0.808 & 0.809 & 0.809 \\
$w{=}1024,d{=}8$ & 0.826 & 0.777 & 0.770 & 0.780 & 0.780 & 0.778 \\
$w{=}1024,d{=}10$& 0.830 & 0.823 & 0.826 & 0.825 & 0.826 & 0.825 \\
\bottomrule
\end{tabular}
\end{center}

Across the regime where sign-$k$NN is already close to the network
accuracy, performance varies only weakly with $k$; changing $k$ from
1 to 20 alters accuracy by at most a few points. The
$w{=}1024,d{=}8$ row sits near the transition region, where the sign
gap is still visible but already small enough to preserve the same
qualitative conclusion.

\section{The Canonical Gauge, the Logit Decomposition, and the
Gate Domination Bound}
\label{app:gauge}

This appendix provides a self-contained derivation of the logit
decomposition (Eq.~\ref{eq:logit-decomp}), explains why the
conditional mean magnitude $\mu_j(s)$ is the unique optimal choice
of reference magnitude---the \emph{canonical gauge}---and gives
the full proof of the gate domination bound
(Theorem~\ref{thm:gate-dom}).

\subsection{Setup and notation}
\label{app:gauge-setup}

Recall from Definition~\ref{def:sign-mag} that every neuron $j$ in
the last hidden layer produces an output
\[
  h_j(x) \;=\; S_j(x)\;\cdot\;M_j(x),
  \qquad
  S_j(x) = \operatorname{sgn}\!\bigl(h_j(x)\bigr) \in \{-1,+1\},
  \quad
  M_j(x) = |h_j(x)| \ge 0.
\]

\textbf{Sign convention.}
We adopt $\operatorname{sgn}(0) = -1$ throughout.  For ReLU, this
assigns the entire inactive region $\{x : z_j(x) \le 0\}$ to the
negative sign ($S_j = -1$), so $S_j \in \{-1,+1\}$ is well defined
for every input.  For smooth sign-preserving activations (SiLU,
GELU, tanh), only the boundary $\{z_j = 0\}$ is affected by this
convention; under an absolutely continuous pre-activation
distribution, this set has probability zero.

The class-$c$ logit is
\begin{equation}\label{eq:logit-raw}
  f_c(x)
  \;=\;
  \sum_{j=1}^{d} w_{cj}\,S_j(x)\,M_j(x) \;+\; b_c\,.
\end{equation}
We write $\mathbf{S}(x) = \bigl(S_1(x),\dots,S_{d}(x)\bigr)
\in \{-1,+1\}^{d}$ for the full last-layer sign code, where $d$
is the last hidden layer width.

Our goal is to split \eqref{eq:logit-raw} into a part that depends
\emph{only} on $\mathbf{S}(x)$ and a remainder.  Any such split
requires choosing, for each neuron~$j$, a \emph{reference magnitude}
$\nu_j(s)$ that stands in for $M_j$ when we know only that
$S_j = s$.  Different choices of $\nu_j$ yield different
decompositions; the question is which choice is best.

\textbf{Tie-breaking convention.}
Throughout this appendix, both the full classifier
$\hat f(x) = \arg\max_c f_c(x)$ and the sign-only classifier
$\hat{\tilde f}(x) = \arg\max_c \tilde f_c(x)$ use a fixed
deterministic tie-breaking rule (e.g.\ lowest index).  We assume
that ties $f_a(x) = f_c(x)$ and
$\tilde f_a(x) = \tilde f_c(x)$ occur with probability zero under
the data distribution~$\mathcal{D}$.  This is a generic
nondegeneracy condition; in particular, the sign-only logits
$\tilde f_c$ depend only on the discrete sign code
$\mathbf{S} \in \{-1,+1\}^{d}$ and therefore take finitely many
values, so ties in the sign-only classifier can occur with positive
probability unless the weights and biases place distinct logit
values on every occupied sign code.  In practice this nondegeneracy
holds for all trained models we tested (see
\cref{app:condition-ties}); when it fails, the bound remains valid
up to an additive correction equal to the tie probability.

\subsection{The general decomposition for an arbitrary gauge}
\label{app:gauge-general}

Fix any reference-magnitude function
$\nu_j\colon \{-1,+1\}\to\mathbb{R}_{\ge 0}$ for each neuron~$j$.
Substituting
$M_j = \nu_j(S_j) + \bigl[M_j - \nu_j(S_j)\bigr]$ into
\eqref{eq:logit-raw}:
\begin{align}
  f_c(x)
  &= \sum_j w_{cj}\,S_j\,\nu_j(S_j) \;+\; b_c
     \;+\;
     \sum_j w_{cj}\,S_j\,\bigl[M_j - \nu_j(S_j)\bigr]
  \nonumber\\[4pt]
  &\equiv\;
     \underbrace{
       \tilde f_c^{(\nu)}(x)
     }_{\text{sign-only logit}}
     \;+\;
     \underbrace{
       R_c^{(\nu)}(x)
     }_{\text{residual}}.
  \label{eq:general-decomp}
\end{align}
The sign-only logit $\tilde f_c^{(\nu)}$ depends on $x$ only through
$\mathbf{S}(x)$; the residual $R_c^{(\nu)}$ absorbs everything
else.
Equation~\eqref{eq:general-decomp} is an \emph{exact identity}---it
holds for every $x$ and every choice of $\nu$, with no
approximation.

\begin{remark}[Codomain of $\nu_j$]
\label{rmk:codomain}
We define
$\nu_j\colon\{-1,+1\}\to\mathbb{R}_{\ge 0}$ (not
$\mathbb{R}_{>0}$) because for ReLU,
$\mu_j(-1) = \mathbb{E}[|h_j|\mid S_j{=}{-1}] = 0$, and the gauge
must be able to take this value.
\end{remark}

\subsection{Why the conditional mean is the canonical gauge}
\label{app:gauge-opt}

For each neuron $j$ and each sign value $s \in \{-1,+1\}$, define
the conditional mean magnitude
\begin{equation}\label{eq:cond-mean}
  \mu_j(s) \;=\; \mathbb{E}\bigl[\,M_j \;\big|\; S_j = s\,\bigr],
\end{equation}
which is well defined whenever $\Pr(S_j = s) > 0$.

\begin{proposition}[MMSE optimality of the conditional mean]
\label{prop:gauge-opt}
For each neuron~$j$, assuming $\mathbb{E}[M_j^2] < \infty$, the
conditional mean $\mu_j(s) = \mathbb{E}[M_j \mid S_j = s]$
minimises the mean squared prediction error
\begin{equation}\label{eq:mmse-obj}
  \mathbb{E}\bigl[\,(M_j - \nu_j(S_j))^2\,\bigr]
\end{equation}
over all functions
$\nu_j\colon\{-1,+1\}\to\mathbb{R}_{\ge 0}$.  The minimiser is
unique for every $s$ with $\Pr(S_j = s) > 0$.
\end{proposition}

\begin{proof}
For each fixed $s \in \{-1,+1\}$ with $\Pr(S_j = s) > 0$, the
bias--variance decomposition gives
\[
  \mathbb{E}\bigl[(M_j - \nu_j(s))^2 \;\big|\; S_j = s\bigr]
  \;=\;
  \operatorname{Var}(M_j \mid S_j{=}s)
  \;+\;
  \bigl(\mu_j(s) - \nu_j(s)\bigr)^2.
\]
The first term is independent of $\nu_j(s)$.  The second is
non-negative and vanishes if and only if $\nu_j(s) = \mu_j(s)$.
Averaging over $S_j$ completes the proof.
\end{proof}

The conditional mean $\mu_j(s)$ is the unique
\emph{minimum mean-squared-error (MMSE) predictor} of $M_j$ from
the single sign bit $S_j = s$.  It extracts exactly the
information about $M_j$ that is available in $S_j$, no more and
no less.

\subsection{Effect of gauge choice on the residual}
\label{app:gauge-residual}

\begin{lemma}[Gauge change is a sign code-dependent shift]
\label{lem:gauge-shift}
For any two gauges $\nu$ and $\mu$,
\begin{equation}\label{eq:gauge-shift}
  R_c^{(\nu)}(x)
  \;=\;
  R_c^{(\mu)}(x) \;+\; \Delta_c\!\bigl(\mathbf{S}(x)\bigr),
  \qquad
  \Delta_c(\mathbf{S})
  \;:=\;
  \sum_j w_{cj}\,S_j\,
    \bigl[\mu_j(S_j) - \nu_j(S_j)\bigr].
\end{equation}
\end{lemma}

\begin{proof}
From $R_c^{(\nu)} = \sum_j w_{cj} S_j (M_j - \nu_j(S_j))$:
\begin{align*}
  R_c^{(\nu)}
  &= \sum_j w_{cj} S_j
     \bigl[(M_j - \mu_j(S_j)) + (\mu_j(S_j) - \nu_j(S_j))\bigr]
  = R_c^{(\mu)} + \Delta_c(\mathbf{S}).
  \qedhere
\end{align*}
\end{proof}

\begin{corollary}[Conditional variance is gauge-invariant]
\label{cor:var-invariant}
Since $\Delta_c(\mathbf{S})$ is a constant conditional on
$\mathbf{S}$:
$\operatorname{Var}[R_c^{(\nu)} \mid \mathbf{S}]
= \operatorname{Var}[R_c^{(\mu)} \mid \mathbf{S}]$
for every gauge $\nu$.
\end{corollary}

\subsection{Activation-specific forms of the canonical gauge}
\label{app:gauge-activations}

\paragraph{ReLU.}
$\mu_j(-1) = 0$,
$\mu_j(+1) = \mathbb{E}[z_j \mid z_j > 0]$.
The sign code recovers the classical gate vector
$g_j = \mathbf{1}[z_j > 0]$ under $S_j = 2g_j - 1$.

\paragraph{SiLU and GELU.}
These attenuate but do not zero out the negative side:
$0 < \mu_j(-1) \ll \mu_j(+1)$, with magnitude asymmetry
$\delta_j \in [0.4, 0.95]$ across tested configurations
(Appendix~\ref{app:asymmetry}).

\paragraph{Tanh.}
Because $\tanh$ is odd and trained networks produce approximately
symmetric pre-activations,
$\mu_j(+1) \approx \mu_j(-1)$ (i.e.\ $\delta_j \approx 0$).
The sign-only logit simplifies to
$\tilde f_c(x) = \sum_j w_{cj}\,\bar\mu_j\,S_j(x) + b_c$
where $\bar\mu_j = \tfrac{1}{2}(\mu_j(+1) + \mu_j(-1))$.

\subsection{The pairwise margin decomposition}
\label{app:gauge-margin}

From \eqref{eq:general-decomp} with the canonical gauge
$\nu = \mu$, the pairwise margin between classes~$a$
and~$c$ decomposes as:
\begin{equation}\label{eq:margin-decomp-app}
  M_{a,c}(x)
  \;=\;
  f_a(x) - f_c(x)
  \;=\;
  \underbrace{
    \tilde f_a(x) - \tilde f_c(x)
  }_{\widetilde{M}_{a,c}(x)}
  \;+\;
  \underbrace{
    R_a(x) - R_c(x)
  }_{R_{a,c}(x)},
\end{equation}
where the sign-only margin
$\widetilde{M}_{a,c}(x)
= \sum_j (w_{a,j} - w_{c,j})\,S_j(x)\,\mu_j(S_j(x))
  + (b_a - b_c)$
depends on $x$ only through the sign code $\mathbf{S}(x)$, and
the pairwise residual
$R_{a,c}(x)
= \sum_j (w_{a,j} - w_{c,j})\,S_j(x)\,
  [M_j(x) - \mu_j(S_j(x))]$.
The full network agrees with the sign-only classifier whenever
$|R_{a,c}(x)| < |\widetilde{M}_{a,c}(x)|$ for every competitor
class~$c$.

\subsection{Full proof of the gate domination bound}
\label{app:gate-dom-proof}

We restate the theorem for convenience and give the complete proof.

\subsubsection*{Notation and definitions}

For a test input $x$ with full-model predicted class
$\hat y = \hat f(x)$ and each competitor $c \neq \hat y$, recall
from the main text:
\begin{itemize}
\item The \textbf{per-neuron vote bias}
  $m_c(x) = \frac{1}{d}[\sum_j (w_{\hat y,j} - w_{c,j})\,
  S_j(x)\,\mu_j(S_j(x)) + (b_{\hat y} - b_c)]$,
  so that $\widetilde{M}_{\hat y,c}(x) = d \cdot m_c(x)$.

\item The \textbf{conditional residual moment}
  $r_c^2(x) = \mathbb{E}[R_{\hat y(x),c}^{\,2}(X) \mid
  \mathbf{S}(X) = \mathbf{S}(x),\, \hat f(X) = \hat y(x)]$.

\item The \textbf{pair-specific} and \textbf{overall
  gate-domination scores}:
  \begin{equation}\label{eq:Z1-def}
    Z_{1,c}(x)
    \;=\;
    \frac{|\widetilde{M}_{\hat y(x),c}(x)|}{r_c(x)}\,,
    \qquad
    Z_1(x)
    \;=\;
    \min_{c \neq \hat y(x)}\; Z_{1,c}(x).
  \end{equation}
\end{itemize}

\subsubsection*{Why we condition on both $\mathbf{S}$ and $\hat f$}

The sign-only margin $\widetilde{M}_{\hat y,c}$ is a deterministic
function of $\mathbf{S}$, but the class pair $(\hat y, c)$ is not:
two inputs sharing the same sign code can have different full-model
predictions (since the full model also depends on magnitudes).
Conditioning on $(\mathbf{S}, \hat f)$ jointly fixes both the
threshold $|\widetilde{M}_{\hat y,c}|$ and the class pair,
making the Markov step valid.

\subsubsection*{The bound}

\begin{theorem}[Gate domination---restated]
\label{thm:gate-dom-full}
Let $\varepsilon_1 = \Pr_x(\hat f(x) \neq \hat{\tilde f}(x))$.
Assume that ties in both classifiers occur with probability zero
under $\mathcal{D}$, and that all conditional residual second
moments $r_c^2(x)$ are finite (guaranteed on bounded-input
datasets with finite network parameters).

\smallskip
\noindent\textbf{(i)~Disagreement bound.}
\begin{equation}\label{eq:eps1-full}
  \varepsilon_1
  \;\le\;
  (C{-}1)\;\mathbb{E}_x\!\bigl[Z_1^{-2}(x)\bigr].
\end{equation}

\smallskip
\noindent\textbf{(ii)~Width scaling.}
If\/ $r_c^2(x) \le d\,\bar\sigma_c^2(x)$ for all
competitors~$c \neq \hat y$, then
\begin{equation}\label{eq:z1-scaling-full}
  Z_1(x) \;\ge\;
  \Bigl(\min_{c \neq \hat y}
  \frac{|m_c(x)|}{\bar\sigma_c(x)}\Bigr)\sqrt{d}\,.
\end{equation}
\end{theorem}

\begin{proof}
The proof has two parts.

\medskip
\noindent\textbf{Part (i): Disagreement bound.}

We proceed in three steps.

\medskip
\noindent\emph{Step 1: Disagreement requires a margin violation.}
Suppose $\hat f(x) \neq \hat{\tilde f}(x)$.  Let
$\hat y = \hat f(x)$ and
$c^* = \hat{\tilde f}(x) \neq \hat y$.  Since
$\hat y = \arg\max_a f_a(x)$, we have
$f_{\hat y}(x) - f_{c^*}(x) > 0$ (strict inequality by the
zero-tie assumption).  Since
$c^* = \arg\max_a \tilde f_a(x)$ and $c^* \neq \hat y$, we have
$\widetilde{M}_{\hat y, c^*}(x) \le 0$.
From
$f_{\hat y} - f_{c^*}
= \widetilde{M}_{\hat y,c^*} + R_{\hat y,c^*}$
with $f_{\hat y} - f_{c^*} > 0$ and
$\widetilde{M}_{\hat y,c^*} \le 0$:
\begin{equation}\label{eq:flip-condition}
  |R_{\hat y,c^*}(x)| \;\ge\; |\widetilde{M}_{\hat y,c^*}(x)|.
\end{equation}
Disagreement therefore implies that at least one class $c^*$
satisfies~\eqref{eq:flip-condition}.

\medskip
\noindent\emph{Step 2: Conditional Markov's inequality.}
Define the $\sigma$-algebra
$\mathcal{G} = \sigma(\mathbf{S}(X),\, \hat f(X))$.
For each fixed class $c \in \{1,\ldots,C\}$, define the event
\[
  B_c
  \;=\;
  \bigl\{\,
    c \neq \hat f(X),\;\;
    |R_{\hat f(X),c}(X)| \ge
    |\widetilde{M}_{\hat f(X),c}(X)|
  \,\bigr\}.
\]
Step~1 shows that
$\{\hat f(X) \neq \hat{\tilde f}(X)\}
\subseteq \bigcup_{c=1}^{C} B_c$,
and hence
\begin{equation}\label{eq:indicator-bound}
  \mathbf{1}\!\bigl[\hat f(X) \neq \hat{\tilde f}(X)\bigr]
  \;\le\;
  \sum_{c=1}^{C} \mathbf{1}_{B_c}.
\end{equation}
We now bound $\mathbb{E}[\mathbf{1}_{B_c} \mid \mathcal{G}]$ for
each~$c$.  Conditional on $\mathcal{G}$, the full-model prediction
$\hat y = \hat f(X)$ is fixed.  If $c = \hat y$, then
$\mathbf{1}_{B_c} = 0$.  If $c \neq \hat y$, the class pair
$(\hat y, c)$ is fixed, the sign-only margin
$|\widetilde{M}_{\hat y,c}|$ is a deterministic constant (being a
function of $\mathbf{S}$ and the fixed $\hat y$), and Markov's
inequality on $R_{\hat y,c}^{\,2}$ gives:
\begin{equation}\label{eq:markov-step}
  \mathbb{E}[\mathbf{1}_{B_c} \mid \mathcal{G}]
  \;=\;
  \mathbf{1}_{\{c \neq \hat y\}}\;
  \Pr\!\bigl(R_{\hat y,c}^{\,2} \ge
  \widetilde{M}_{\hat y,c}^{\,2}
  \;\big|\; \mathcal{G}\bigr)
  \;\le\;
  \mathbf{1}_{\{c \neq \hat y\}}\;
  \frac{\mathbb{E}[R_{\hat y,c}^{\,2} \mid \mathcal{G}]}
       {\widetilde{M}_{\hat y,c}^{\,2}}
  \;=\;
  \mathbf{1}_{\{c \neq \hat y\}}\;
  Z_{1,c}^{-2}(X).
\end{equation}
This step requires only finite conditional second moments.

\medskip
\noindent\emph{Step 3: Integrate and combine.}
Taking expectations of~\eqref{eq:indicator-bound} and applying
tower property with~\eqref{eq:markov-step}:
\begin{align}
  \varepsilon_1
  &\;=\;
  \mathbb{E}\!\bigl[
    \mathbf{1}[\hat f(X) \neq \hat{\tilde f}(X)]
  \bigr]
  \;\le\;
  \sum_{c=1}^{C}
  \mathbb{E}\!\bigl[\mathbf{1}_{B_c}\bigr]
  \;=\;
  \sum_{c=1}^{C}
  \mathbb{E}\!\bigl[
    \mathbb{E}[\mathbf{1}_{B_c} \mid \mathcal{G}]
  \bigr]
  \nonumber\\
  &\;\le\;
  \mathbb{E}\!\Biggl[
    \sum_{c \neq \hat f(X)} Z_{1,c}^{-2}(X)
  \Biggr]
  \;\le\;
  (C{-}1)\;\mathbb{E}\!\bigl[Z_1^{-2}(X)\bigr],
  \label{eq:final-bound}
\end{align}
where the last inequality uses
$Z_1^{-2} = \max_{c \neq \hat y} Z_{1,c}^{-2} \ge Z_{1,c}^{-2}$
for each~$c$, and the sum over $c \neq \hat f(X)$ has exactly
$C - 1$ terms.

\medskip
\noindent\textbf{Part (ii): Width scaling.}

By definition,
$|\widetilde{M}_{\hat y,c}(x)| = d\,|m_c(x)|$.
The assumption $r_c^2(x) \le d\,\bar\sigma_c^2(x)$ gives
$r_c(x) \le \bar\sigma_c(x)\sqrt{d}$.
Therefore:
\[
  Z_{1,c}(x)
  \;=\;
  \frac{|\widetilde{M}_{\hat y,c}(x)|}{r_c(x)}
  \;\ge\;
  \frac{d\,|m_c(x)|}{\bar\sigma_c(x)\sqrt{d}}
  \;=\;
  \frac{|m_c(x)|}{\bar\sigma_c(x)}\,\sqrt{d}\,.
\]
Minimising over $c \neq \hat y$ gives
$Z_1(x) \ge (\min_{c} |m_c|/\bar\sigma_c)\sqrt{d}$.
\end{proof}

\subsection{Exponential rate under sub-Gaussian tails}
\label{app:subgauss-bound}

The main theorem uses only Markov's inequality, yielding a
polynomial bound in $Z_1$.  Under stronger tail assumptions, the
rate improves to exponential.  This is not required for the main
theorem; we include it to contextualize the empirical contour in
Figure~\ref{fig:gate-dom} uses the exponential form.

\begin{assumption}[Conditional sub-Gaussianity]
\label{ass:subg}
Fix $\mathcal{G} = \sigma(\mathbf{S}(X), \hat f(X))$ and a
competitor $c \neq \hat f(X)$.  Conditional on $\mathcal{G}$, the
pairwise residual $R_{\hat f(X),c}$ is sub-Gaussian with variance
proxy $\sigma_c^2$:
\[
  \Pr\!\bigl(|R_{\hat f(X),c}| \ge t \;\big|\;
  \mathcal{G}\bigr)
  \;\le\;
  2\exp\!\Bigl(-\frac{t^2}
  {2\,\sigma_c^2}\Bigr),
  \qquad \forall\; t \ge 0.
\]
The proxy $\sigma_c^2$ may differ from the conditional second
moment $r_c^2$; in general $\sigma_c^2 \ge
\operatorname{Var}(R_{\hat y,c} \mid \mathcal{G})$.
\end{assumption}

Under Assumption~\ref{ass:subg}, the proof of Part~(i)
carries through with the sub-Gaussian tail replacing Markov at
Step~2, yielding:
\begin{equation}\label{eq:eps1-subgauss}
  \varepsilon_1
  \;\le\;
  2(C{-}1)\;\mathbb{E}_x\!\Bigl[
    \exp\!\bigl(-\tfrac{1}{2}\,(Z_1^{(\exp)})^2(x)\bigr)
  \Bigr],
\end{equation}
where $Z_{1,c}^{(\exp)}(x) =
|\widetilde{M}_{\hat y,c}(x)|/\sigma_c(x)$ and
$Z_1^{(\exp)}(x) = \min_{c \neq \hat y} Z_{1,c}^{(\exp)}(x)$.
This becomes non-vacuous at
$Z_1^{(\exp)} > \sqrt{2\log(2(C{-}1))} \approx 2.4$ for $C = 10$.
The diagnostics in \cref{app:conditions-summary} indicate that sub-Gaussian
concentration holds approximately at moderate widths.

\subsection{The expanded form of $Z_1$ for ReLU networks}
\label{app:z1-expanded}

For ReLU networks, inactive neurons contribute zero magnitude
($\mu_j(-1) = 0$), so only the active set
$J = \{j : S_j(x) = +1\}$ with $d_{\mathrm{act}} = |J|$
contributes.  The pair-specific score admits the approximation
\begin{equation}\label{eq:Z1-expanded}
  Z_{1,c}(x)
  \;\approx\;
  \frac{\sqrt{d_{\mathrm{act}}(x)}\;\rho_{\hat y,c}(x)}
       {\hat\sigma_r(x)\,\sqrt{1 + \hat\rho_m}}\,,
\end{equation}
where $\rho_{\hat y,c}$ is the normalised gate-only margin and
$\hat\sigma_r$ the empirical per-neuron relative-magnitude standard
deviation.  Width provides $\sqrt{d_{\mathrm{act}}}$ averaging;
depth provides larger $\rho_{\min}$.
For general sign-preserving activations, the sum runs over all $d$
neurons and $\mu_j(S_j)$ replaces $\mu_j(+1)$.

\section{Empirical Diagnostics Supporting the Theorem Conditions}
\label{app:theorem-conditions}

The disagreement bound in Theorem~\ref{thm:gate-dom} requires
only finite conditional second moments and nondegeneracy (no ties).
The width-scaling statement (Part~ii) additionally requires the
residual-scaling condition $r_c^2 \le d\,\bar\sigma_c^2$.
The interpretation that $Z_1$ grows with width further relies on
positive vote bias ($m_c > 0$).
We provide empirical diagnostics supporting each of these conditions
across the full SVHN width--depth grid (FFNN, SiLU,
12~depths $\times$ 10~widths, 120~models).
All diagnostics are computed from two passes: canonical-gauge
statistics $\mu_j(s)$, $\sigma_j^2(s)$ from the training split,
and residual diagnostics from the test split.
Figure~\ref{fig:theorem-conditions} summarises the results.

\begin{figure}[htb!]
\centering
\includegraphics[width=\textwidth]{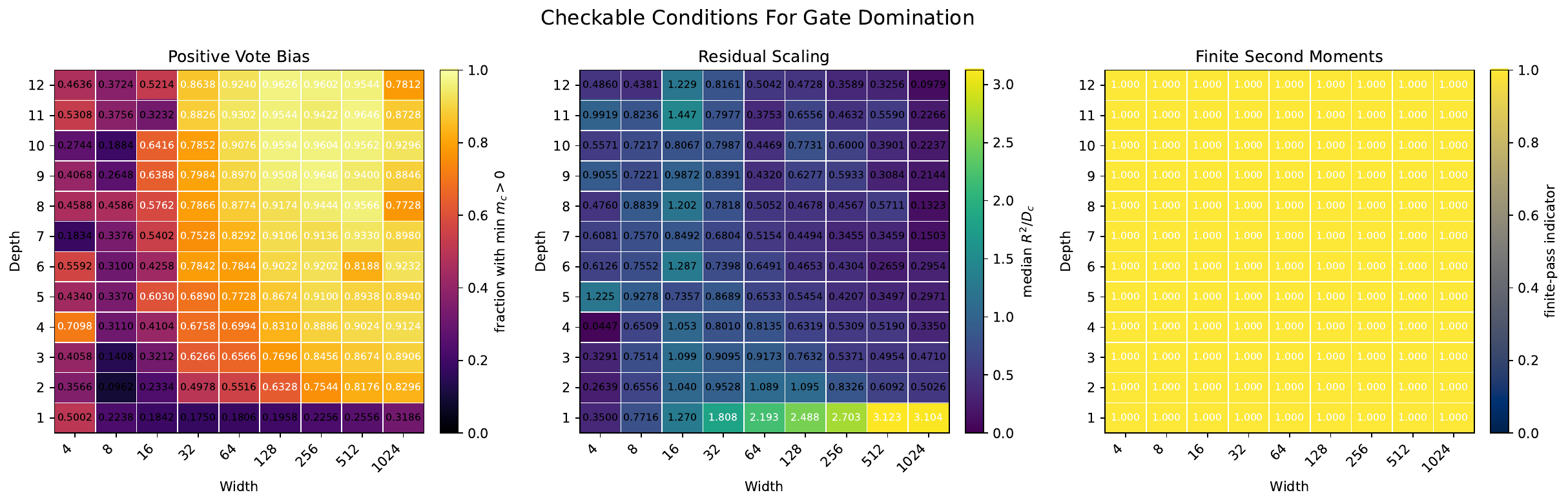}
\caption{%
\textbf{Empirical diagnostics for the gate-domination theorem
conditions}
(FFNN, SiLU, SVHN; widths 4--1024, depths 1--12).
\textbf{Left:}~Positive vote bias.  Each cell shows the fraction
of test inputs for which $m_c(x) > 0$ for all competitor classes.
In the overparameterized regime (depth $\ge 4$, width $\ge 64$),
this fraction exceeds $0.80$; the sign-only and full classifiers
agree on exactly these inputs.
\textbf{Centre:}~Residual scaling diagnostic.  Each cell shows the
median of the per-input ratio $R_{\hat y,c}^2(x) / D_c(x)$, where
$D_c = \sum_j (w_{\hat y,j} - w_{c,j})^2 \sigma_j^2(S_j)$ is
the diagonal (uncorrelated) variance proxy.
Under approximate conditional uncorrelatedness and a CLT
approximation over $d$ neurons, this ratio is approximately
$\chi^2(1)$ with median~${\approx}0.45$.  The observed values are
consistent with this prediction in the transitional regime
(depth~${\ge}3$, width~${\ge}32$), supporting the residual-scaling
condition.
\textbf{Right:}~Finite second moments.  All cells show $1.0$:
every gauge value, residual, and logit is finite across the entire
grid, as guaranteed by bounded SVHN inputs and finite network
parameters.}
\label{fig:theorem-conditions}
\end{figure}

\subsection{Positive vote bias ($m_c > 0$)}
\label{app:condition-vote-bias}

For each test input $x$, we compute the sign-only margin
$\widetilde{M}_{\hat y,c}(x)$ for every competitor $c \neq \hat y$
and check whether all margins are positive, i.e.\ whether
$m_c(x) > 0$ for every competitor.
Figure~\ref{fig:theorem-conditions}~(left) displays the fraction of
test inputs satisfying this condition across the width--depth grid.

\paragraph{Depth is the bottleneck.}
At depth~1, the all-positive fraction is below $0.35$ across all
widths: the sign-only classifier disagrees with the full model on
most inputs.  This is consistent with the theorem's prediction:
at depth~1, no compositional feature extraction has occurred, so
per-neuron vote contributions are approximately unbiased
($m_c \approx 0$) and no amount of width can compensate.

\paragraph{The transition is sharp.}
Between depth~2 and depth~4, the fraction jumps from ${\sim}0.50$
to ${\sim}0.85$ at width~$128$.  By depth~$6$ and width~$\ge 64$,
the fraction exceeds $0.90$ across all tested configurations.

\paragraph{Vote bias tracks agreement exactly.}
In our diagnostic data, the fraction of inputs where the sign-only
classifier agrees with the full model equals the fraction where
all vote biases are positive, in every row of the grid.
Under the no-ties assumption, this is expected:
$m_c(x) > 0$ for all $c \neq \hat y$ if and only if
$\hat{\tilde f}(x) = \hat y$.
This confirms that the mechanism identified by the
theorem---disagreement occurs when the residual overwhelms a
negative sign-only margin---is the actual mechanism at work.

\subsection{Residual scaling ($r_c^2 \le d\,\bar\sigma_c^2$)}
\label{app:condition-residual}

The theorem's condition requires that the conditional residual
moment $r_c^2(x) = \mathbb{E}[R_{\hat y,c}^2 \mid \mathbf{S},
\hat f]$ grows at most linearly in width.  We do not estimate
$r_c^2(x)$ exactly (this would require grouping test inputs by
their full sign code, which is infeasible in high dimensions).
Instead, we use the diagonal proxy
\begin{equation}\label{eq:diag-proxy}
  D_c(x)
  \;=\;
  \sum_{j=1}^d (w_{\hat y,j} - w_{c,j})^2\,
  \sigma_j^2\!\bigl(S_j(x)\bigr),
\end{equation}
which equals $r_c^2(x)$ when the per-neuron residual contributions
are exactly uncorrelated conditional on $(\mathbf{S}, \hat f)$.
We examine the per-input ratio
$\rho_c(x) = R_{\hat y,c}^2(x) \,/\, D_c(x)$.

\paragraph{Calibration benchmark.}
If the per-neuron residual contributions were exactly uncorrelated
conditional on $(\mathbf{S}, \hat f)$, then by the central limit
theorem (over $d$ neurons) the standardised residual
$R_{\hat y,c}/\sqrt{D_c}$ would be approximately standard normal,
so $\rho_c \sim \chi^2(1)$ with median $\approx 0.455$ and
mean $= 1$.  This benchmark requires both approximate
uncorrelatedness \emph{and} the CLT approximation.

\paragraph{Empirical results.}
Figure~\ref{fig:theorem-conditions}~(centre) shows the median of
$\rho_c$ across the grid.  In the transitional regime (depth $\ge 3$,
width $32$--$256$), the median is $0.4$--$0.9$, consistent with
the $\chi^2(1)$ prediction.  This supports the interpretation that
the diagonal proxy $D_c$ correctly captures the typical residual
scale, and is therefore consistent with the theorem's
residual-scaling condition $r_c^2 = O(d)$.
The mean of $\rho_c$ ranges from $1.5$ to $3.0$, reflecting a
heavy right tail: a small fraction of inputs have
$R^2 \gg D_c$.  If the condition holds with an effective constant
$C_0 \approx 2$--$3$ rather than $C_0 = 1$, the width-scaling
bound becomes
$Z_1 \ge (|m_c|/(C_0^{1/2}\bar\sigma_c))\sqrt{d}$: the
pre-factor changes by at most $\sqrt{3} \approx 1.7$, but the
$\sqrt{d}$ rate is unaffected.

\paragraph{Cross-correlation diagnostic.}
We directly measure pairwise correlations between per-neuron
residual contributions by computing the off-diagonal entries of the
sample correlation matrix of the standardised per-unit residuals.
The mean absolute off-diagonal correlation $\hat\rho_m$ satisfies
$\hat\rho_m \le 0.10$ for width $\le 256$ across all depths,
supporting the approximate uncorrelatedness interpretation.
At width $\ge 512$, $\hat\rho_m$ can increase
(up to ${\sim}0.6$) because most neurons become effectively
always-on or always-off, reducing the number of variable units.
This degradation occurs precisely where $Z_1$ is already large and
the bound is conservative.

At depth~1, the median ratio exceeds $2.0$ at width $\ge 128$,
indicating stronger correlations.  However, depth~1 already fails
on vote bias ($m_c$ not positive), so the scaling interpretation
is inapplicable regardless.

\subsection{Finite second moments}
\label{app:condition-finite}

Figure~\ref{fig:theorem-conditions}~(right) confirms that every
cell in the grid has value $1.0$: all gauge values
$\mu_j(s)$, $\sigma_j^2(s)$, all residuals $R_{\hat y,c}(x)$,
all logits, and all magnitudes are finite.  This is guaranteed by
the bounded support of SVHN images propagated through networks with
finite parameters.
No special verification is needed beyond confirming the absence of
numerical overflow, which we do programmatically.

\subsection{Nondegeneracy (ties)}
\label{app:condition-ties}

The theorem assumes that ties in both classifiers occur with
probability zero.  We check this with tolerance $10^{-8}$: the
fraction of test inputs where the top two logits differ by less
than $10^{-8}$ is exactly $0.0$ for both the full model and the
sign-only classifier in every cell of the grid.

\subsection{Summary}
\label{app:conditions-summary}

The main theorem (Theorem~\ref{thm:gate-dom}) requires only finite
conditional second moments and nondegeneracy, both of which hold
trivially across the entire grid.  The diagnostics supporting the
width-scaling interpretation and the positive-vote-bias corollary
are as follows:
\begin{itemize}
\item \textbf{Positive vote bias:} the fraction of inputs with
  $m_c > 0$ for all competitors exceeds $0.85$ for depth $\ge 4$
  and width $\ge 64$, and exceeds $0.95$ for depth $\ge 6$.
  The sign-only and full classifiers agree on precisely these
  inputs.

\item \textbf{Residual scaling:} the per-input diagnostic ratio
  $R^2/D_c$ has median $0.4$--$0.9$ (consistent with the
  $\chi^2(1)$ benchmark under approximate uncorrelatedness and a
  CLT approximation) and mean $1.5$--$3.0$ in the transitional
  regime.  Cross-correlations satisfy $\hat\rho_m \le 0.10$ for
  width $\le 256$.  These diagnostics are consistent with the
  theorem's condition $r_c^2 = O(d)$, though they do not
  constitute a direct estimate of the conditional moment $r_c^2(x)$.

\item \textbf{Finite second moments:} trivially satisfied on
  bounded-input datasets with finite network parameters.
\end{itemize}
The sub-Gaussian condition required for the exponential remark
(\cref{app:subgauss-bound}) holds in the transitional regime
(width $\le 128$, tail ratios ${\le}1.3$ at $t{=}2$) and degrades
at larger widths where the main Markov bound is already sufficient.


\section{Activation Asymmetry and the Limits of Gate Domination}
\label{app:asymmetry}

The gate-domination bound (Theorem~\ref{thm:gate-dom-full}) is
\emph{activation-universal}: it holds for any sign-preserving
activation under the finite-residual-moment and zero-tie conditions
stated in the theorem.  What \emph{does} depend on the activation is
the size of the gate-domination score $Z_1$---and therefore how
useful the bound is in practice.

This appendix makes that dependence precise.  We introduce the
\emph{magnitude asymmetry ratio} $\delta_j$
(\S\ref{app:asymmetry-def}), show how it enters the sign-only margin
(\S\ref{app:asymmetry-margin}), and use it to explain why $Z_1$ is
large for ReLU/SiLU/GELU but remains modest for tanh
(\S\ref{app:asymmetry-z1}), with quantitative validation across 240
trained models (\S\ref{app:tanh-vs-silu}).  We then discuss the
distinct mechanism by which residual architectures produce weak $Z_1$
(\S\ref{app:resnet-z1}).  Finally, we explain why our repulsion
regularizer $\mathcal{L}_{\mathrm{rep}}$ uses a sigmoid gate proxy
for ReLU-family activations but a tanh sign proxy for tanh networks
(\S\S\ref{app:reg-gate}--\ref{app:reg-tanh}).

\subsection{The magnitude asymmetry ratio}
\label{app:asymmetry-def}

Let $\mu_j(s) = \mathbb{E}[M_j \mid S_j = s]$ be the
sign-conditional mean magnitude from the canonical gauge
(Appendix~\ref{app:gauge}, Eq.~\ref{eq:cond-mean}).  For any
neuron~$j$ with $\mu_j(+1) + \mu_j(-1) > 0$, define
\begin{equation}\label{eq:delta-def}
  \bar\mu_j
  \;:=\;
  \tfrac{1}{2}\bigl(\mu_j(+1) + \mu_j(-1)\bigr),
  \qquad
  \delta_j
  \;:=\;
  \frac{\mu_j(+1) - \mu_j(-1)}
       {\mu_j(+1) + \mu_j(-1)}
  \;\in\; [-1,\,1].
\end{equation}
For dead neurons with $\mu_j(+1) = \mu_j(-1) = 0$ we set
$\delta_j = 0$ by convention.  The identity
\begin{equation}\label{eq:mu-delta}
  \mu_j(S_j)
  \;=\;
  \bar\mu_j\,\bigl(1 + \delta_j\, S_j\bigr)
\end{equation}
is immediate from the definitions.

The quantity $\delta_j$ measures how strongly a neuron's expected
magnitude depends on its sign.  Three regimes arise:
\begin{enumerate}
\item \textbf{ReLU} ($\delta_j = 1$ exactly).  Inactive neurons
  have $\mu_j(-1) = 0$; the sign bit completely determines whether
  the neuron contributes.

\item \textbf{SiLU / GELU} ($\delta_j$ large and positive
  empirically, typically $0.4$--$0.95$).  The negative side is
  attenuated but not zeroed.  Whether
  $\mu_j(-1) \ll \mu_j(+1)$ holds in a given network depends on
  the pre-activation distribution, not only on the activation
  (see~\S\ref{app:gauge-activations}).

\item \textbf{Tanh} ($\delta_j \approx 0$ empirically).  Because
  $\tanh$ is odd, $|\tanh(z)| = |\tanh(-z)|$ for all~$z$.  Under
  the approximately symmetric pre-activation distributions produced
  by batch-normalised networks, this yields
  $\mu_j(+1) \approx \mu_j(-1)$ and hence $\delta_j \approx 0$.
  This is an empirical statement about trained networks, not a
  consequence of oddness alone
  (see the discussion in~\S\ref{app:gauge-activations}).
\end{enumerate}

\subsection{How asymmetry enters the sign-only margin}
\label{app:asymmetry-margin}

Under the canonical gauge, the sign-only logit is
$\tilde f_c(x) = \sum_j w_{cj}\,S_j(x)\,\mu_j(S_j(x)) + b_c$.
Substituting \eqref{eq:mu-delta} and using $S_j^2 = 1$:
\begin{align}
  \tilde f_c(x)
  &= \sum_j w_{cj}\,\bar\mu_j\,S_j\,(1 + \delta_j S_j) + b_c
  \nonumber\\[2pt]
  &= \underbrace{%
       \sum_j w_{cj}\,\bar\mu_j\,S_j
     }_{\text{sign-linear term}}
     \;+\;
     \underbrace{%
       \sum_j w_{cj}\,\bar\mu_j\,\delta_j
     }_{\text{constant asymmetry offset}}
     \;+\; b_c\,.
  \label{eq:ftilde-delta}
\end{align}
The pairwise sign-only margin therefore decomposes as
\begin{equation}\label{eq:margin-delta}
  \widetilde{M}_{a,c}(x)
  \;=\;
  \underbrace{%
    \sum_j (w_{a,j} - w_{c,j})\,\bar\mu_j\,S_j(x)
  }_{\widetilde{M}^{\mathrm{lin}}_{a,c}(x)}
  \;+\;
  \underbrace{%
    \sum_j (w_{a,j} - w_{c,j})\,\bar\mu_j\,\delta_j
  }_{\Delta^{(\delta)}_{a,c}}
  \;+\;
  (b_a - b_c).
\end{equation}

Equation~\eqref{eq:margin-delta} is the key observation.
Asymmetry affects gate domination entirely through the
\emph{numerator} of the pair-specific score
$Z_{1,c}(x) = |\widetilde{M}_{\hat y,c}(x)| \,/\,
\sigma_{\hat y,c}(\mathbf{S}(x))$
from Eq.~\eqref{eq:Z1-def}; the structure of the disagreement
bound itself (Theorem~\ref{thm:gate-dom-full}) is unchanged.

\paragraph{Interpretation.}
The sign-linear term $\widetilde{M}^{\mathrm{lin}}_{a,c}$ depends on
the input through $\mathbf{S}(x)$.  The asymmetry offset
$\Delta^{(\delta)}_{a,c}$ is a constant that depends only on the
learned weights and the activation-specific asymmetry profile.
When this offset is aligned with the output-weight differences
(i.e.\ has the same sign as $\widetilde{M}^{\mathrm{lin}}_{a,c}$
for correctly classified inputs), it enlarges the sign-only margin
and increases $Z_1$.  When $\delta_j \approx 0$ for all~$j$, the
offset vanishes and the sign-only margin is governed entirely by the
sign-linear term.  There is no general monotone lower bound on
$|\widetilde{M}_{a,c}|$ in terms of the mean asymmetry
$\bar\delta = d_L^{-1}\sum_j \delta_j$ alone, because the effect
depends on alignment with the class-weight structure.

\subsection{Why $Z_1$ is strong for SiLU/ReLU but weak for tanh}
\label{app:asymmetry-z1}

The gate-domination bound
$\varepsilon_1 \le 2(C{-}1)\,\mathbb{E}_x[\exp(-Z_1^2/2)]$
is activation-universal.  Asymmetry does not alter the bound; it
alters how large $Z_1$ typically is.

For \textbf{ReLU} ($\delta_j = 1$), inactive neurons contribute zero
to the sign-only logit and the margin is dominated by the active set.
The asymmetry offset $\Delta^{(\delta)}_{a,c}$ is maximal.
For \textbf{SiLU and GELU}, $\delta_j$ remains strongly positive
empirically: positive-sign neurons contribute substantially more than
negative-sign neurons, so the offset in~\eqref{eq:margin-delta}
amplifies the sign-only margin.  This makes large $Z_1$ values
typical in wide, deep feedforward networks with these activations.

For \textbf{tanh}, trained networks typically satisfy
$\mu_j(+1) \approx \mu_j(-1)$, so $\delta_j \approx 0$ and the
offset is negligible.  The sign-only margin then reduces
approximately to the sign-linear term:
\[
  \widetilde{M}_{a,c}(x)
  \;\approx\;
  \sum_j (w_{a,j} - w_{c,j})\,\bar\mu_j\,S_j(x)
  \;+\; (b_a - b_c).
\]
Because $\tanh$ is odd, every neuron contributes with comparable
magnitude regardless of sign.  The margin is therefore the difference
of two nearly balanced sums of mixed-sign terms, which tends to
cancel and yields a typical scale much smaller than in the
asymmetric case.  Consequently, $Z_1$ remains modest even when
the sign code is causally dominant under intervention and highly informative.

\subsection{Empirical validation: tanh versus SiLU}
\label{app:tanh-vs-silu}

We now verify the predictions of~\eqref{eq:margin-delta}
quantitatively.  We train 240 FFNN models on CIFAR-10 with each
activation (12~depths $\times$ 10~widths $\times$ 2~seeds) and
compute all decomposition quantities from a single forward pass over
the test set.

\paragraph{Notation.}
To compare across models we report the \emph{median gate-only margin}
\begin{equation}\label{eq:gamma-tilde}
  \tilde\gamma_{\mathrm{med}}
  \;=\;
  \operatorname{median}_{x}\;
  \min_{c \neq \hat{\tilde y}(x)}\;
  \widetilde{M}_{\hat{\tilde y},c}(x),
\end{equation}
i.e.\ the median over test points of the minimum pairwise sign-only
margin from~\eqref{eq:margin-delta}.  We abbreviate this as
$\tilde\gamma$ in Table~\ref{tab:tanh-vs-silu}.  The residual noise
$\hat\sigma_r$ is the empirical residual-scale estimator associated
with the diagonal proxy in \cref{app:condition-residual}; $\bar\delta$ is the layer-mean
asymmetry $d_L^{-1}\sum_j \delta_j$; and $Z_1$ is the median of the
overall gate-domination score from Eq.~\eqref{eq:Z1-def}.

Table~\ref{tab:tanh-vs-silu} shows representative configurations;
Figure~\ref{fig:asymmetry-heatmaps} displays the full heatmaps
across the width--depth grid, and
Figure~\ref{fig:delta-vs-margin} shows that $\bar\delta$ is a strong
predictor of the gate-only margin across models.

\begin{table}[t]
\centering
\caption{Magnitude asymmetry and gate-domination diagnostics for
SiLU versus tanh on CIFAR-10 (FFNN, selected configurations,
averaged over 2 seeds).}
\label{tab:tanh-vs-silu}
\small
\begin{tabular}{@{}rr|cccc|cccc@{}}
\toprule
& & \multicolumn{4}{c|}{\textbf{SiLU}}
  & \multicolumn{4}{c}{\textbf{Tanh}} \\
Width & Depth
  & $\bar\delta$ & $\hat\sigma_r$ & $\tilde\gamma$ & $Z_1$
  & $\bar\delta$ & $\hat\sigma_r$ & $\tilde\gamma$ & $Z_1$ \\
\midrule
  64 &  6 & 0.94 & 8.39 & 1.46 & ${\sim}4$
         & 0.007 & 0.84 & 0.86 & ${\sim}0.6$ \\
  64 & 12 & 0.93 & 6.53 & 1.44 & ${\sim}6$
         & 0.003 & 0.89 & 1.00 & ${\sim}0.5$ \\
 256 &  6 & 0.86 & 3.68 & 1.40 & ${\sim}5$
         & 0.000 & 0.87 & 1.26 & ${\sim}0.6$ \\
 256 & 12 & 0.71 & 1.82 & 2.53 & ${\sim}9$
         & 0.001 & 0.75 & 1.17 & ${\sim}0.5$ \\
1024 &  6 & 0.55 & 1.19 & 2.08 & ${\sim}7$
         & 0.001 & 0.65 & 0.98 & ${\sim}0.5$ \\
1024 & 12 & 0.41 & 0.53 & 3.01 & ${\sim}10$
         & 0.001 & 0.61 & 0.85 & ${\sim}0.5$ \\
4096 &  6 & 0.37 & 0.49 & 2.71 & ${\sim}9$
         & 0.000 & 0.60 & 0.92 & ${\sim}0.5$ \\
4096 & 12 & 0.66 & 0.29 & 2.89 & ${\sim}10$
         & 0.000 & 0.59 & 0.95 & ${\sim}0.4$ \\
\bottomrule
\end{tabular}
\end{table}

\paragraph{Key observations.}

\begin{enumerate}
\item \textbf{Tanh has $\bar\delta \approx 0$ everywhere.}
  Across all 120 (width, depth) configurations, the median
  $|\mu_j(+1) - \mu_j(-1)|$ is $0.0014$, and 95 of 120 cells have
  gap $\le 0.005$
  (Figure~\ref{fig:asymmetry-heatmaps}, top row).
  This confirms that the magnitude is effectively independent of the
  sign in trained tanh networks, as predicted by approximate
  pre-activation symmetry.

\item \textbf{SiLU has large $\bar\delta$.}
  At small widths, $\bar\delta > 0.9$ (a near-ReLU regime where the
  negative side is strongly suppressed).  At large widths and depths,
  $\bar\delta$ decreases to ${\sim}0.4$--$0.7$ as pre-activations
  concentrate closer to zero, but remains far above zero
  (Figure~\ref{fig:asymmetry-heatmaps}, top row).

\item \textbf{$\hat\sigma_r$ is smaller for tanh than for SiLU.}
  SiLU has larger $\hat\sigma_r$ in 118 of 120 cells
  (Figure~\ref{fig:asymmetry-heatmaps}, bottom row).
  The residual noise is \emph{not} the bottleneck for tanh---the
  signal (margin) is.

\item \textbf{Gate-only margins are 2--3$\times$ larger for SiLU.}
  At (4096, 12), SiLU achieves
  $\tilde\gamma_{\mathrm{med}} = 2.89$ versus tanh's $0.95$
  (Figure~\ref{fig:asymmetry-heatmaps}, middle row;
  Figure~\ref{fig:delta-vs-margin}).
  This gap, combined with the $\hat\sigma_r$ ratio, explains the
  order-of-magnitude difference in $Z_1$.

\item \textbf{Despite small $Z_1$, tanh sign codes are causally
  necessary.}
  Section~\ref{sec:causal} showed that tanh networks on SVHN retain
  $78\%$ accuracy (from $80\%$ baseline) under total magnitude
  replacement, and collapse to chance under $10\%$ sign corruption.
  The sign code carries the classification signal; the $Z_1$
  diagnostic cannot certify this because the conditional-mean gauge
  extracts no asymmetry from an approximately odd activation.
\end{enumerate}

\begin{figure}[t]
\centering
\includegraphics[width=\textwidth]{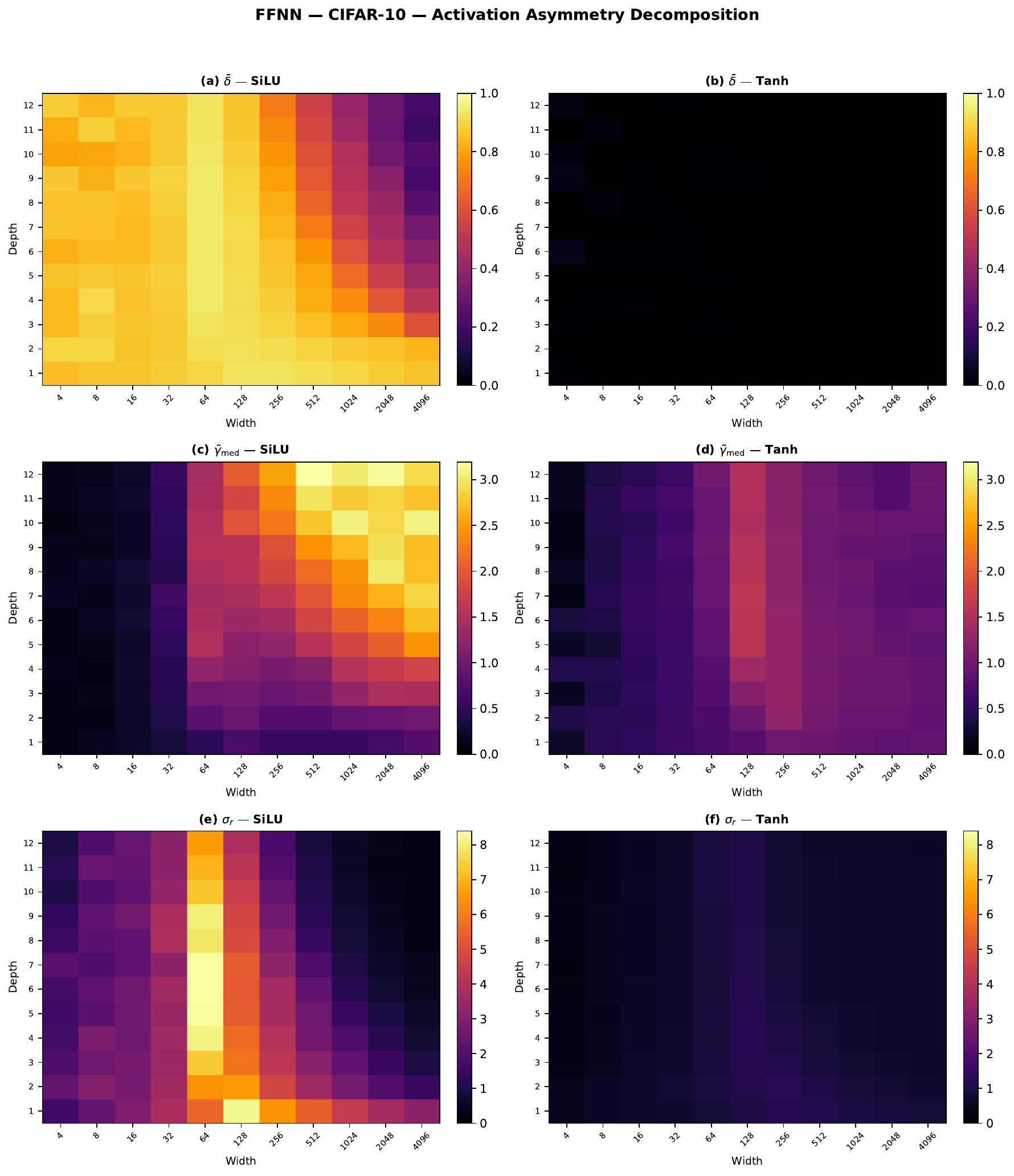}
\caption{\textbf{Activation asymmetry explains the $Z_1$ gap between
SiLU and tanh.}  All panels show FFNN on CIFAR-10, widths 4--4096,
depths 1--12, averaged over 2 seeds.
\textbf{Top row:}~the magnitude asymmetry $\bar\delta$
(Eq.~\ref{eq:delta-def}) is uniformly large for SiLU and uniformly
near zero for tanh.
\textbf{Middle row:}~gate-only margins $\tilde\gamma_{\mathrm{med}}$
(Eq.~\ref{eq:gamma-tilde}) are 2--3$\times$ larger for SiLU,
because the asymmetry offset in~\eqref{eq:margin-delta} amplifies
the contribution of correctly-routed neurons.
\textbf{Bottom row:}~residual noise $\hat\sigma_r$ is comparable or
\emph{larger} for SiLU, confirming that the $Z_1$ gap is entirely
due to the margin, not the noise.}
\label{fig:asymmetry-heatmaps}
\end{figure}

\begin{figure}[t]
\centering
\includegraphics[width=0.6\textwidth]{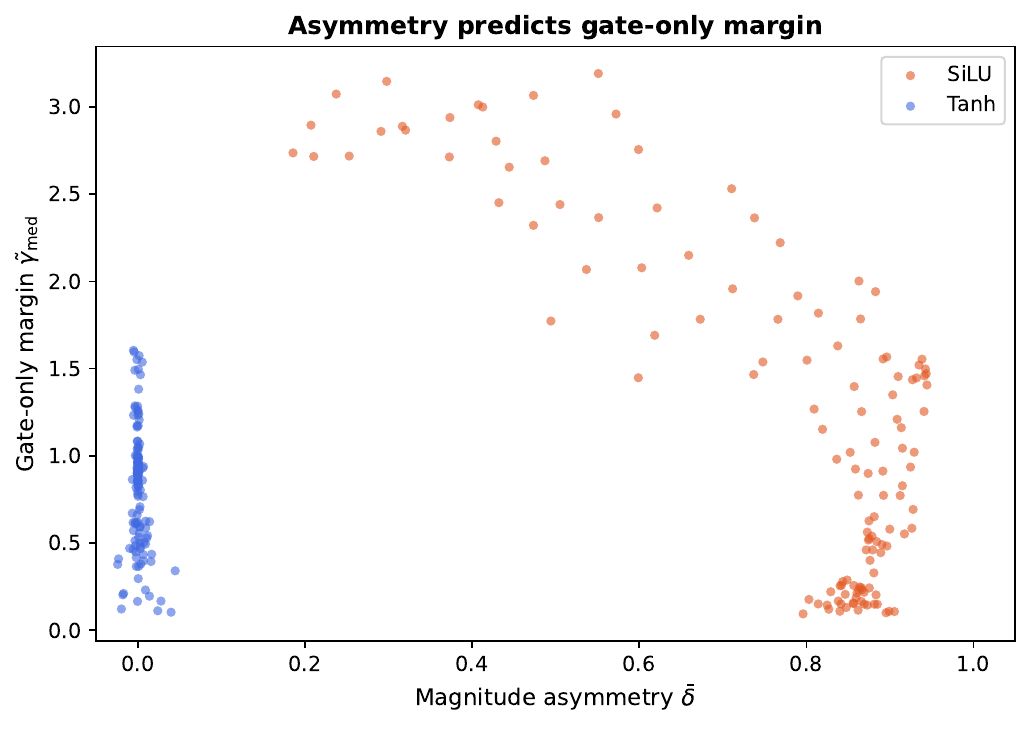}
\caption{\textbf{Asymmetry predicts gate-only margin.}  Each point
is one trained FFNN on CIFAR-10.  SiLU models (orange) span
$\bar\delta \in [0.3, 0.95]$ with margins up to $3.2$; tanh models
(blue) cluster at $\bar\delta \approx 0$ with margins below $1.6$.}
\label{fig:delta-vs-margin}
\end{figure}

\subsection{Why residual architectures can have weak $Z_1$}
\label{app:resnet-z1}

Residual architectures present a distinct mechanism for weak $Z_1$.

In a residual block the output is
$h^{(\ell+1)} = \varphi\bigl(h^{(\ell)} + F^{(\ell)}(h^{(\ell)})
\bigr)$,
and the sign--magnitude decomposition at the block output mixes the
skip connection $h^{(\ell)}$ (which carries its own sign--magnitude
structure) with the branch output $F^{(\ell)}$.  The last-layer
magnitude $|h_j^{(L)}|$ therefore depends on the entire history of
skip connections, creating dependencies that the single-layer
conditional mean $\mu_j(S_j)$ cannot capture.  This inflates the
effective residual scale $\sigma_{\hat y,c}(\mathbf{S})$ in the
denominator of $Z_{1,c}$, rather than shrinking the numerator as in
the tanh case.

The two mechanisms are thus complementary:
\begin{enumerate}
\item \textbf{Odd activations (tanh):} weak $Z_1$ primarily because
  the sign-only margin $|\widetilde{M}_{a,c}|$ is small (the
  asymmetry offset vanishes).
\item \textbf{Residual architectures:} weak $Z_1$ primarily because
  the residual scale $\sigma_{\hat y,c}$ is large (skip-path
  magnitude dependencies inflate the denominator).
\end{enumerate}
In both cases, the causal intervention results
(Section~\ref{sec:causal}) provide the more robust, architecture-
and activation-universal evidence that signs carry the decisive
classification signal.  For ResNet-18 with SiLU, replacing every
magnitude with a random draw preserves $93$--$95\%$ accuracy across
all four tested activations; flipping $10\%$ of sign bits collapses
accuracy to chance.

A multi-layer decomposition that conditions on the sign codes of all
preceding blocks---rather than using a single last-layer conditional
mean---would yield tighter residual estimates and potentially larger
$Z_1$ values.  We leave such multi-layer certificates to future work,
noting that the single-layer $Z_1$ is best understood as a
diagnostic for feedforward architectures, while the causal
intervention framework provides the architecture-universal evidence.

\begin{figure}[htb!]
\centering
\includegraphics[width=\textwidth]{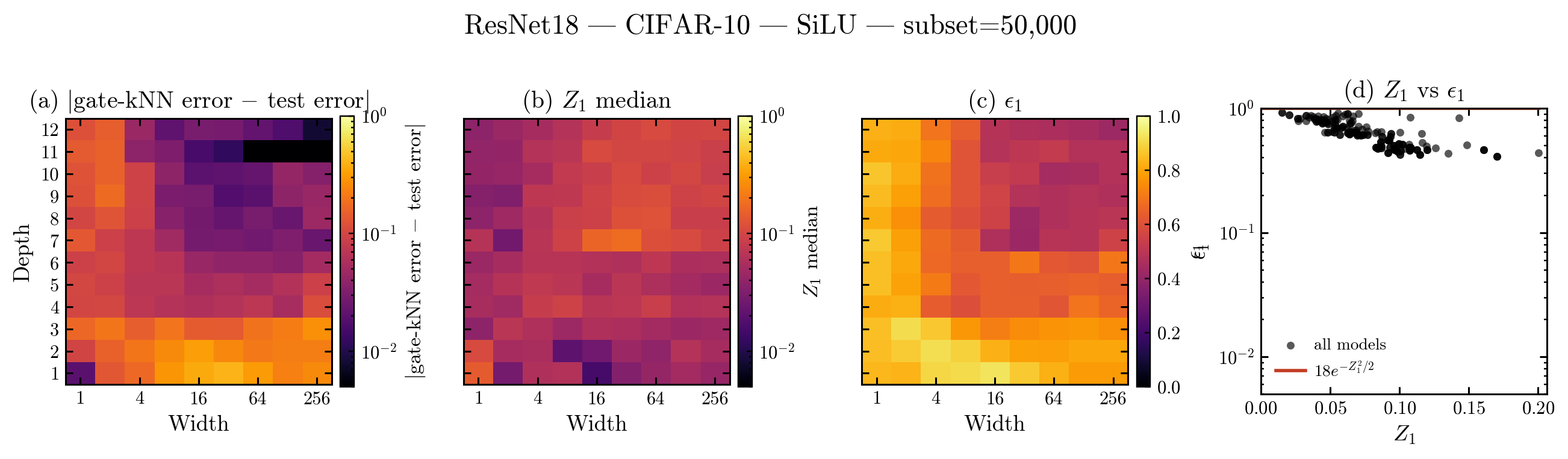}
\caption{Theoretical quantities across the depth-width grid for a ResNet on Cifar-10 with SiLU activation.}
\label{fig:resnet-z1-diagnostics}
\end{figure}

\subsection{Choice of repulsion regularizer: gate proxy vs.\ sign
proxy}
\label{app:reg-gate}

The repulsion loss $\mathcal{L}_{\mathrm{rep}}$ (Eq.~\ref{eq:repulsion})
is designed to destroy same-class sharing of the \emph{dominant
computational pathway}.  Which pathway is dominant depends on the
activation's asymmetry profile, so the correct continuous proxy for
the regularizer depends on the activation.

\subsubsection*{ReLU, SiLU, and GELU: sigmoid gate proxy}

For these activations, $\delta_j$ is large and positive: neurons
with $S_j = +1$ contribute far more strongly than neurons with
$S_j = -1$.  The dominant computational pathway is therefore the
\emph{active set} $\{j : g_j = 1\}$, where
$g_j = \mathbf{1}[z_j > 0]$.  We regularise the continuous gate
proxy
\begin{equation}\label{eq:gate-proxy}
  g(x) \;=\; \sigma\!\bigl(z(x)/\tau\bigr)
  \;\in\; (0,1)^D,
\end{equation}
which for small temperature $\tau$ approximates the binary gate.
Penalising $\bigl(g(x_i)^\top g(x_j)\bigr)^2$ for same-class pairs
directly drives their active-set overlap toward zero.

\paragraph{Why gate repulsion is stronger than sign repulsion for
non-odd activations.}
Two inputs with $D$ neurons, each with roughly $D/2$ active, can
have \emph{orthogonal} sign codes in $\{-1,+1\}^D$ while still
sharing approximately $D/4$ co-active neurons.  Sign code
orthogonality treats agreements on the active side and agreements on
the inactive side symmetrically, but only the active side carries
substantial forward-pass signal when $\delta_j$ is large.
Gate repulsion, by contrast, can drive the co-active count all the
way to zero, destroying the shared computational channel entirely.

\paragraph{Empirical confirmation.}
On SVHN (FFNN, width~$1024$, depth~$8$, SiLU), sigmoid-gate
repulsion achieves the intended effect: training accuracy exceeds
$97\%$, test accuracy collapses to approximately $40\%$, and
leave-one-out purity at $k{=}5$ drops to $0.34$.  Matched
sign-based repulsion at the same nominal strength either
destabilises training (SiLU/GELU/ReLU, due to steeper gradients from
the tanh proxy) or leaves enough active-set overlap for partial
generalization to survive.

\subsection{Tanh networks: sign proxy}
\label{app:reg-tanh}

Tanh is the important exception.  Because $\delta_j \approx 0$ in
trained tanh networks, the forward pass does not privilege an active
set: positive-sign and negative-sign neurons contribute with
comparable magnitude.  A sigmoid gate proxy would impose an
artificial positive-vs-nonpositive split that is not aligned with
the tanh network's geometry.

For tanh models we therefore use the sign proxy
\begin{equation}\label{eq:sign-proxy}
  q(x) \;=\; \tanh\!\bigl(z(x)/\tau\bigr)
  \;\in\; (-1,1)^D,
\end{equation}
which for small $\tau$ approximates the binary sign code
$S_j = \operatorname{sgn}(z_j) \in \{-1,+1\}$.  This is the correct
matched intervention: when the activation treats both sides
approximately symmetrically, the relevant binary structure is the
sign code itself, not a privileged active set.

\paragraph{Consistency with the rest of the paper.}
The distinction between gate proxy and sign proxy arises
\emph{only} in the training regularizer.  All theoretical
diagnostics (the gate-domination score $Z_1$, the error chain, the
sub-Gaussian bound), all empirical measurements (sign-kNN, LOO
purity), and all causal interventions (sign flip, magnitude
replacement) are reported throughout the paper in the common binary
sign-code framework $S_j \in \{-1,+1\}$.  The two frameworks are
related by the bijection $S_j = 2g_j - 1$; the gate representation
$g_j \in \{0,1\}$ appears only inside the regularizer where it
provides the stronger intervention for non-odd activations.

\subsection{Summary}
\label{app:asymmetry-summary}

\begin{enumerate}
\item The gate-domination theorem
  (Theorem~\ref{thm:gate-dom-full}) is activation-universal;
  asymmetry does not change the bound itself.

\item Asymmetry enters through the sign-only margin via the
  decomposition~\eqref{eq:margin-delta}: large positive $\delta_j$
  amplifies the margin; $\delta_j \approx 0$ eliminates the
  asymmetry offset and leaves only the sign-linear term.

\item For ReLU, SiLU, and GELU (large $\delta_j$), gate domination
  is readily achieved in wide, deep feedforward networks.

\item For tanh ($\delta_j \approx 0$) and for residual architectures
  (large residual scale), $Z_1$ can remain modest even when the
  sign code is causally dominant under intervention and informationally sufficient
  (Table~\ref{tab:tanh-vs-silu};
  Figures~\ref{fig:asymmetry-heatmaps}
  and~\ref{fig:delta-vs-margin}).

\item The training regularizer is matched to the activation's
  geometry: sigmoid gate proxy~\eqref{eq:gate-proxy} for
  ReLU-family activations (where the active set is the dominant
  channel), tanh sign proxy~\eqref{eq:sign-proxy} for tanh
  (where no such asymmetric split exists).
\end{enumerate}

Table~\ref{tab:asymmetry-summary} collects the evidence across
activations and architectures.

\begin{table}[h]
\centering
\caption{Summary of gate-domination evidence across activations and
architectures.}
\label{tab:asymmetry-summary}
\small
\begin{tabular}{@{}lccc@{}}
\toprule
Configuration
  & Sign causal (\S\ref{sec:causal})
  & Sign-kNN tracks accuracy (\S\ref{sec:signknn})
  & $Z_1$ certifies \\
\midrule
FFNN + SiLU/GELU/ReLU
  & \checkmark & \checkmark & \checkmark \\
FFNN + Tanh
  & \checkmark & \checkmark
  & \texttimes\; (small $\delta$) \\
ResNet-18 + SiLU
  & \checkmark & \checkmark
  & \texttimes\; (skip connections) \\
ViT + GELU
  & \checkmark & \checkmark & not tested \\
\bottomrule
\end{tabular}
\end{table}
\newpage



\end{document}